%% file: main.tex
\documentclass[journal]{IEEEtran}
\usepackage{amsmath,amsfonts,amssymb}
\usepackage{algorithm}
\usepackage{array}
\usepackage[caption=false,font=normalsize,labelfont=sf,textfont=sf]{subfig}
\usepackage{textcomp}
\usepackage{stfloats}
\usepackage{url}
\usepackage{verbatim}
\usepackage{graphicx}
\usepackage{cite}
\usepackage{breqn}
\usepackage[noend]{algpseudocode}
\usepackage{textcomp}
\usepackage{xcolor}
\usepackage{bm}
\usepackage{bbm}
\usepackage{xcolor}
\usepackage{accents}
\usepackage{dsfont}
\usepackage{lipsum}
\usepackage{multirow, booktabs}
\usepackage{etoolbox}
\usepackage{amsthm}
\usepackage[english]{babel}
\usepackage{bbm}

\newtheorem{theorem}{Theorem}

\usepackage[pdftex]{hyperref}

\newcommand{\ubar}[1]{\underaccent{\bar}{#1}}

\def\BibTeX{{\rm B\kern-.05em{\sc i\kern-.025em b}\kern-.08em
    T\kern-.1667em\lower.7ex\hbox{E}\kern-.125emX}}
    
\begin{document}

\title{Efficient Interpretable Nonlinear Modeling for Multiple Time Series
\thanks{The work in this paper was supported by the SFI Offshore Mechatronics grant 237896/O30 and the IKTPLUSS INDURB grant 270730/O70 from the Norwegian Science Foundation, and the VALIDATE project grant 101057263 from the EU HORIZON-RIA. Part of this work was presented in INTAP 2021 \cite{luism} and ISMODE 2022 \cite{roy2022joint}}
}
\author{
     \IEEEauthorblockN{Kevin Roy\IEEEauthorrefmark{1}, 
     Luis Miguel Lopez-Ramos\IEEEauthorrefmark{2}, 
     and Baltasar Beferull-Lozano\IEEEauthorrefmark{1}\IEEEauthorrefmark{2}
     \\
     \IEEEauthorblockA{\IEEEauthorrefmark{1}WISENET Center, Department of ICT, University of Agder, Grimstad, Norway}
     \\
     \IEEEauthorblockA{\IEEEauthorrefmark{2}
     Simula Metropolitan Center for Digital Engineering, Oslo, Norway}
     \\
 {kevin.roy@uia.no; luis@simula.no; baltasar.beferull@uia.no; baltasar}@simula.no}
}




\maketitle

\input{IEEE_Journal_TSP/abstract}
\input{IEEE_Journal_TSP/introduction}
\input{IEEE_Journal_TSP/preliminaries}

\input{IEEE_Journal_TSP/Modelling_assumption}
\input{IEEE_Journal_TSP/modelling}
\input{IEEE_Journal_TSP/experiments}

\input{IEEE_Journal_TSP/conclusion}

\input{IEEE_Journal_TSP/appendixA}

\input{IEEE_Journal_TSP/appendixB}

\vspace{12pt}

\bibliographystyle{IEEEtran}
\bibliography{bibliography}

\vfill

\end{document}

%% file: IEEE_Journal_TSP/abstract.tex
\begin{abstract} 
Predictive linear and nonlinear models based on kernel machines or deep neural networks have been used to discover dependencies among time series. This paper proposes an efficient nonlinear modeling approach for multiple time series, with a complexity comparable to linear vector autoregressive (VAR) models while still incorporating nonlinear interactions among different time-series variables. The modeling assumption is that the set of time series is generated in two steps: first, a linear VAR process in a latent space, and second, a set of invertible and Lipschitz continuous nonlinear mappings that are applied per sensor, that is, a component-wise mapping from each latent variable to a variable in the measurement space. The VAR coefficient identification provides a topology representation of the dependencies among the aforementioned variables. The proposed approach models each component-wise nonlinearity using an invertible neural network and imposes sparsity on the VAR coefficients to reflect the parsimonious dependencies usually found in real applications. To efficiently solve the formulated optimization problems, a custom algorithm is devised  
combining proximal gradient descent, stochastic primal-dual updates, and projection to enforce the corresponding constraints.
Experimental results on both synthetic and real data sets show that the proposed algorithm improves the identification of the support of the VAR coefficients in a parsimonious manner while also improving the time-series prediction, as compared to the current state-of-the-art methods.
\end{abstract}


%% file: IEEE_Journal_TSP/introduction.tex
\section{Introduction}
\label{introduction}
In many engineering fields, such as financial engineering, signal analysis from sensor networks, brain signal processing, and interconnected systems in water networks and the oil and gas sector, to mention a few, determining the dependencies among several interconnected systems is an important task. Many of these scenarios include the measurement and storage of several time series, often obtained from sensors that are associated with other sensor variables of the same underlying physical process being observed. Such relationships may be represented as a graph structure that consists of nodes and edges, where each node represents a time series, and the edges or arcs between nodes typically represent a function expressing the dependency between the time series associated with the two connected nodes.

Note that such large-scale systems can become very complex in terms of the number of dependencies between different sensors. The set of relationships between them, usually referred to as the “topology” of the sensor network, can also be interpreted by human operators and can vary depending on the various control actions happening in the system. The methods for learning these dependencies are of considerable significance \cite{giannakis2018topology}. The interdependencies between different sensor variables are often modeled using a graph representation \cite{dong}, which is helpful for tasks such as prediction \cite{zaman2020online}, change point detection \cite{lopez2018dynamic}, and data compression \cite{chatarjee}, among others.

Within the plethora of methods that have been proposed to identify dependencies between interconnected systems, Granger causality (GC)~\cite{grangerc} is a widely used paradigm. The GC quantifies the degree to which the history of one time series helps predict the future of another time series. More specifically, a time series is said to be Granger-caused by another if the optimal prediction error of the former decreases when the history of the latter time series is considered \cite{tirsobakth}. There are alternative causality definitions based on the vector autoregressive (VAR) model, which represents interactions between variables with linear or nonlinear functions~\cite{sparsitybasu,tank2021neural,grangercausalitygoebel}.

The VAR model has been proven useful in multiple applications involving topology identification \cite{giannakis}. VAR causality is determined from the support of VAR matrix parameters and is equivalent to GC under certain conditions  \cite{tirsobakth}. In the case of a linear VAR, \cite{tirsobakth}, the previous time samples of one time series have an impact on the future of the other series that is modeled as a linear equation representing a causal linear filter. The causality estimates in VAR models can be made scalable to high-dimensional settings using regularizers that enforce sparsity over the VAR parameters\cite{sparsitybasu}.

Other linear models, such as structural equation models (SEM) and structural VAR (SVAR) models, are often utilized to learn linear causal dependencies among connected time series \cite{giannakis}. SEM does not take into account temporal dependencies, while VAR and SVAR both capture delayed interactions. Topology identification in linear VAR models has been extensively researched \cite{giannakis2018topology,tirsobakth,ioannidis2019semiblind}.

In real-world applications, such as brain networks and industrial sensor data networks, employing linear models may result in inconsistent assessments of causal relationships \cite{tank} because the underlying physical process might have nonlinear interactions. Investigation of nonlinear models is a growing area of research since linear models often struggle to capture nonlinear relationships or dependencies.

Although there is a large body of research on nonlinear causal discovery\cite{Marinazzo,stephan,shen2018online,money2021online,rohanjoshin,shen2019nonlinear}, only a small number of studies \cite{tank2021neural,bussmann2020neural} have successfully used Deep Learning (DL) to identify causal relationships in time series. Deep neural networks are used to model temporal dependencies and interactions between the variables under the GC framework. Regarding nonlinear extensions to the VAR model, functions in reproducing kernel Hilbert spaces (RKHS) are used in \cite{shen2018online,money2021online} to identify nonlinear dependencies by mapping variables to a higher-dimensional Hilbert space where dependencies are linear. Theoretically, DL methods enable the modeling of nonlinear causal interactions \cite{tank2021neural}, providing high expressive power, but their flexibility has a drawback: since DNNs, in general, are black-box approximators, it makes it more challenging to comprehend and interpret the causal links that are learned, despite being the main goal of causal structure learning. In addition, these techniques are typically computationally expensive. 


This work proposes a method that enables interpretable modeling of nonlinear interactions using feed-forward invertible neural networks (INNs) as the main tool to take nonlinearities into account. The fundamental premise of the proposed model is that a set of time series is assumed to be generated by a VAR process in a latent space and that each time series is then observed using a nonlinear, component-wise, monotonically increasing (thus invertible) function represented by an INN. It avoids the black-box nature of many DL-based architectures. We impose sparsity-inducing penalties on the VAR coefficients to improve interpretability and enhance the capacity to manage limited data in the high-dimensional scenario. In this paper, we detail two different formulations with two different levels of complexity.

Linear VAR-causality is often used as the modeling tool to test for GC \cite{lutkepohl2005}. The notion of causality that this paper works with is based on the linear interactions in the latent space, as will be detailed in Sec.~\ref{preliminaries}. Due to the invertible nonlinearities, there is a one-to-one correspondence between variable values in the measurement and latent spaces, and therefore when a causal connection is identified in the linear model in the latent space, it can be deemed present with the same strength between the corresponding pair of variables in the measurement space. 



The first algorithm explicitly uses the inverse, having a fitting cost function based on the prediction error in the sensor signal domain. On the other hand, the second algorithm does not require the inverse calculation, having a cost function based on the prediction error in the latent space, which will be proven to be a bound on the former cost function. The second algorithm has lower computational complexity than the first algorithm, requiring constant memory needs for each iteration, making it suitable for sequential and big-data or high-dimensional scenarios. 

We also empirically validate the performance of these two algorithms, and compare it with currently existing DL-based nonlinear models through extensive tests on synthetic and real data sets.
First, simulations are carried over synthetically-generated signals, namely a nonlinear VAR (matching the modeling assumption) for different values of the lag order $P$, and data generated by the nonlinear Lorenz-96 model \cite{lorentz96} for different values of the force constant $F$, showing that our interpretable approach identifies the graph of nonlinear interactions. Finally, we also evaluate the performance of our methods using real data from a sensor network from a use case in the offshore oil and gas industry. 

The contributions of the present paper can be summarized as follows:\begin{itemize}
    \item A comprehensive description of the proposed modeling assumption that allows inference of nonlinear dependency graphs among any set of time series.
    \item Design of an inference algorithm based on explicit inversion of the functions mapping between the latent and measurement space (formulation A). 
    \item A theoretical result stating under which conditions the prediction MSE in the latent space is an upper bound of the prediction MSE in the measurement space, motivating the formulation of an alternative algorithm.
    \item Derivation of an inference algorithm based on MSE minimization in the latent space (formulation B) which addresses the same modeling assumption and is computationally more efficient. 
    \item Experimental results validating both proposed algorithms, establishing that formulation B outperforms formulation A, and comparing their prediction and topology-identification performance against state-of-the-art GC inference algorithms based on DL.
\end{itemize}


The conference versions of this work present a preliminary version of formulation A with the derivation of the necessary gradients via implicit differentiation in \cite{luism}, and incorporating sparsity-enforcing regularization (including numerical results showcasing its impact on topology identification) in \cite{roy2022joint}. 

The rest of the paper is organized as follows: Sec. \ref{preliminaries} introduces background on linear andnonlinear topology identification. Sec. \ref{modelling_assumption} describes the modeling assumption in detail. Sec. \ref{Modelling_problem_formulation} describes the two formulations and the algorithms to solve them. Sec. \ref{Simulation_experiments} contains simulation and experiments on real and synthetic data sets comparing the strength of our algorithms with other state-of-the-art methods. Finally, Sec. \ref{Simulation_experiments} concludes the paper.

%% file: IEEE_Journal_TSP/preliminaries.tex
\section{Preliminaries}
\label{preliminaries}
After outlining the notion of linear causality graphs, this section reviews how these graphs can be identified by formulating an optimization problem. Then, the basics of the nonlinear causality graphs problem are described.
\subsection{Linear causality Graphs}
Consider a collection of $N$ sensors providing $N$ time series $\left\{y_n[t]\right\}_{n= 1}^N, \, t = 0,1, \ldots, T$, $t \in \mathbb{Z}$, where $y_n[t]$ denotes the measurement of the $n^{th}$ sensor at time $t$. A causality graph $\mathcal{G} \triangleq(\mathcal{V}, \mathcal{E})$ is a directed graph where the $n^{th}$ vertex in $\mathcal{V}=\{1, \ldots, N\}$ is identified with the $n^{th}$ time series ${\left\{y_n[t]\right\}}_{t=0}^T$ and there is a directed edge from $n^{\prime}$ to $n$ (i.e. $\left(n, n^{\prime}\right)\in \mathcal{E}$ ) if and only if ${\left\{y_{n^{\prime}}[t]\right\}}_{t=0}^T$ causes ${\left\{y_n[t]\right\}}_{t=0}^T$. The notion of causality that we deal with in this work is VAR-causality, which is equivalent to GC under certain conditions, and it is easy to obtain from a VAR model. A $P^{th}$-order linear VAR model can be formulated as

\begin{align} \label{eq:var_matrix}
      y[t]=\sum_{p=1}^{P} A^{(p)} y[t-p]+u[t],  \quad \quad P \leq t \leq T
\end{align}
where $y[t]=[y_1[t],\dots,y_N[t]]^T$, $A^{(p)}\ \in\ R^{N\times N}$ and $p = 1,\dots, P$, are respectively the matrices of VAR parameters, $T$ is the observation time period, and $u[t]={[u}_1[t],\dots,u_N[t]]^\top$ is a vector innovation process typically modeled as a Gaussian, temporally-white random process. Letting $a_{n,n^\prime}^{(p)}$ denote the $(n,n^\prime)$ entry of the matrix $A^{(p)}$, \ref{eq:var_matrix} takes the form:
 \begin{align} \label{eq:var_scalar}
       y_n[t] = &\sum_{n^\prime=1}^{N}\sum_{p=1}^{P}{a_{n,n^\prime}^{(p)}y_{n^\prime}}[t-p]+\ u_n[t], \quad P \leq t \leq T \\
        = &\sum_{n^{\prime} \in \mathcal{N}(n)} \sum_{p=1}^P a_{n, n^{\prime}}^{(p)} y_{n^{\prime}}[t-p]+u_n[t]
	\end{align}
for $n = 1,\ldots, N$.  where  $\mathcal{N}(n) \triangleq\left\{n^{\prime}: a_{n, n^{\prime}} \neq 0_P\right\} $ and $a_{n,n^\prime}  = [a_{n,n^\prime}^{(1)},....,\ a_{n,n^\prime}^{(p)}]^{T}$ is the impulse response from node $n^\prime$ to node $n$; this will be a zero vector when there is no edge from node $n^\prime$ to node $n$. Thus, $\left\{y_{n^{\prime}}[t]\right\} \text { VAR-causes }\left\{y_n[t]\right\} \text { if } a_{n, n^{\prime}} \neq 0_P $. 
It therefore holds that the set of directed edges is $\mathcal{E} \triangleq\left\{\left(n, n^{\prime}\right): \boldsymbol{a}_{n, n^{\prime}} \neq \mathbf{0}_P\right\}$, and the in-neighborhood of node $n$, denoted as $\mathcal{N}(n)$, contains all the nodes causing (having a non-zero impulse response connected towards) node $n$.

The problem of identifying a linear VAR causality model boils down to estimating the VAR coefficient matrices $\{A^{(p)}\}_{p=1}^P$ given the observations $\{y[t]\}_{t=0}^{T-1}$. 
To quantify the strength of these dependencies, a weighted graph can be constructed by assigning e.g. the weight $\left\|\boldsymbol{a}_{n, n^{\prime}}\right\|_2$ to the edge $\left(n, n^{\prime}\right)$.

The VAR coefficients can be learned by solving a minimization problem with a least-squares loss. Moreover, models with a reduced number of nonzero parameters entail a reduced number of edges are preferable as they are more parsimonious, motivating the following sparsity-enforced optimization problem with a Lasso-type penalty \cite{high_dim_lozano}:
\begin{align} \label{eq:linearvar}
\min_{\{A^{(p)}\}_{p=1}^P} & \sum_{t=P}^T\left\|y[t]-\left(\sum_{p=1}^P A^{(p)} (y[t-p])\right)\right\|_2^2 \nonumber \\
& +\lambda \sum_{p=1}^P  \sum_{n=1}^N \sum_{n^{\prime}=1}^N\left|{a_{n, n^{\prime}}}^{(p)}\right|
\end{align}
where $|.|$ denotes the absolute value. The hyper-parameter $\lambda>0$ controls the level of sparsity enforced by the $l_1$ norm of the coefficients. The objective function \eqref{eq:linearvar} is non-differentiable which will be considered when designing the iterative algorithms to solve this problem, as we explain in Sec \ref{Modelling_problem_formulation}.

\subsection{Nonlinear modeling}

As stated in Sec \ref{introduction}, time-series collections in many practical applications usually exhibit nonlinear interactions, thus a linear VAR model is insufficient for capturing the nonlinear data dependencies. In the most general nonlinear case, VAR models are not capable of identifying nonlinear dependencies, and their prediction error in real-world scenarios is high. Each data variable $y_n[t]$ can be represented as a nonlinear function of multiple multivariate data time series as follows:
     \begin{align} \label{eq:nonvar}
	y_n[t] =   h_n(y_{t-1},\ldots, y_{t-P}) + u_n[t],
	\end{align}
    where $y_{t-p} = [y_1[t-p],y_2[t-p], \dots,y_N[t-p]]^{\top}$, $p \in [1,P]$ and $h_n(\cdot)$ is a nonlinear function.

    
   

    
The model in \eqref{eq:nonvar} has two main drawbacks: the first one is that there are infinitely many nonlinear functions that can fit a finite set of data points. The second one is that, even if $h_n(\cdot)$ could be identified, there is no clear criterion in the literature to determine an interpretable  graph that allows us to identify which key variables are affecting another variable from such a set of nonlinear functions. In Sec. \ref{modelling_assumption} we present the nonlinear model that we consider to circumvent the aforementioned drawbacks.

%% file: IEEE_Journal_TSP/Modelling_assumption.tex
\section{Interpretable Nonlinear Model}
\label{modelling_assumption}
In this work, we are restricting  the nonlinear function to be learned to  belong to a subset of possible nonlinear functions which comes in between the linear model and the general nonlinear model in terms of complexity.

We aim to design an interpretable nonlinear model. Notice that the linear VAR model is interpretable because its coefficients represent a notion of additive influence of each variable on any another as it can be seen in \eqref{eq:var_matrix}. Since we seek a model having the advantage of identifying dependencies, our model should have a structure resembling that of a VAR model. Linearity renders VAR models not capable of identifying nonlinear dependencies, and their prediction error in real-world scenarios is high. Therefore, the desiderata here is a model which gives low prediction error as compared to linear models while retaining interpretability.

\begin{figure}[ht]
\hspace{-0.6cm}
\includegraphics[width=1.2\columnwidth]{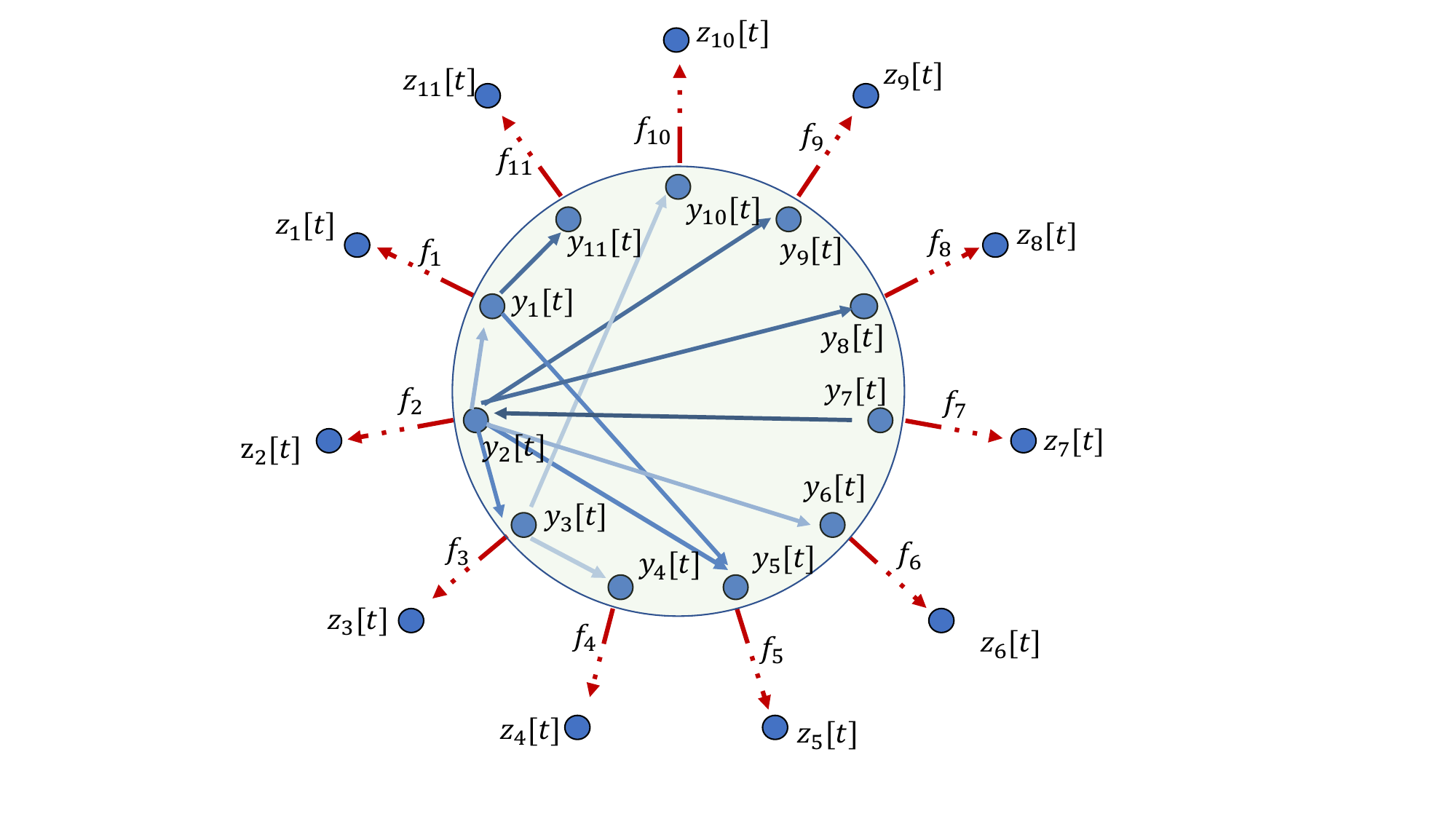}
\caption{{Causal dependencies are assumed linear in the latent space (green circle). In this model, the available sensor data corresponds to the output of the nonlinear functions $\{f_i\}_{i=1}^N$. }
} 
\label{fig:model_architecture}
\end{figure}
To achieve this, we propose a modeling assumption stating that a collection of time series is generated through a VAR process in a latent space, and then each time-series $\left\{z_i[t]\right\}$ is observed in a measurement space through a per-sensor nonlinear, monotonically increasing (and thus invertible) function ($f_i$) connecting $\left\{y_i[t]\right\}$ with $\left\{z_i[t]\right\}$. Each nonlinear function $f_i$  associated with each time series $z_i[t]$ is generally different. The concept is depicted in Fig. \ref{fig:model_architecture}: the green circle represents the  vector space where the latent variables lie, among which the dependencies are linear. The area outside the circle represents the space where the actual sensor measurements $\{z_i[t]\}$ lie. The blue lines represent the linear dependencies between time series in the latent space. Darker blue lines depict stronger dependencies between pairs of sensors. The red line from each time series or sensor represents the corresponding measurement space transformation. 

Let $f:\mathbb{R^N} \rightarrow \mathbb{R^N}$ denote a vector function such that 
$[ f(x)]_i = f_i(x_i)$ where $f_i$ is the nonlinear function associated with each sensor. With this definition, a collection of nonlinearly related time series is assumed to be generated as
\begin{equation}
\label{eq:zfromy}
    z[t] =  f(y[t]),
\end{equation}
where $y[t]$ is generated according to \eqref{eq:var_matrix}.

Since there is a one-to-one mapping between $z[t]$ and $y[t]$ (defined by the bijective mapping $\bf f$), we can say that if $y_n[t]$ VAR causes $y_m[t]$, then clearly $z_n[t]$ causes $z_m[t]$. Moreover, given the nonlinearity of each per-sensor mapping, the latter dependency is nonlinear. The structure of the nonlinear dependency graph among the signals in $z[t]$ is the same as that for the signals in $y[t]$. Therefore, the modeling assumption introduced in this section allows a criterion for inferring a nonlinear causality graph among any set of time series.

Once the model for the generative process is specified, we can express  the problem statement as follows:
Given a set of sensor measurement data given by multiple time series $z[t]$ in a given time interval $[0, T]$, our goal is to identify the linear parameters $\{a_{n,n^\prime}^{(p)}\}$ in the latent space and the vector nonlinear function $f$. In Sec. \ref{Modelling_problem_formulation}, we formally describe the problem formulation and the techniques to infer the aforementioned parameters.

%% file: IEEE_Journal_TSP/modelling.tex
\section{Problem formulation and algorithm design}
\label{Modelling_problem_formulation}
Here, we provide a rigorous problem formulation and the design of algorithms under the modeling assumption described in Sec. \ref{modelling_assumption} resulting in a complexity that is comparable to that of the linear VAR model, while accounting for nonlinear interactions. We consider two different problem formulations; while the direct approach described in Sec. \ref{section_1_modelling} is relatively straightforward, the one in Sec. \ref{section_2_modelling} is advantageous in terms of computation and accuracy.


The problem statement at the end of Sec. \ref{modelling_assumption} requires learning nonlinear functions, and in order to do that, it is necessary to parameterize the functions. The parameterization of nonlinear transformations is different in Sec. \ref{section_1_modelling}  and Sec. \ref{section_2_modelling}.

\subsection{Explicit function inversion-based inference }
\label{section_1_modelling}
A first approach can be based on inferring the nonlinear function parameters by formulating an optimization problem that directly penalizes the difference between the predicted and actual values  of the time series in the measurement space.

The problem can be expressed as follows: given a total of $T$  observations $\{{z}[t]\}_{t=0}^{T-1}$ from the time series, learn the nonlinear transformation $f$ and the parameters $\{A^{(p)}\}_{p=1}^P$ of the underlying linear model in the latent space.

To infer $f$, each $f_i$  is parameterized as a NN layer with $M$ units indexed by $j$ representing the function:
\vspace{-2mm}
\begin{align} 
\label{eq:f_model}
    f_{i}\left(y_{i}\right)= b_{i} + \sum_{j=1}^{M} \alpha_{ij} h\left(w_{ij}y_{i}-k_{ij}\right)
    \vspace{-3mm}
\end{align}
Where $h(\cdot)$ is a monotonically increasing activation function, for example, a sigmoid function; and the parameters to be learned: $\{\alpha_{ij}, w_{ij}, k_{ij}\}_j, b_{i}$ are collected in the vector $\theta_i$: 
{
\small $$\theta_{i}=\left[\begin{array}{l}\alpha_{ i} \\  w_i\\ k_{ i} \\ b_{i}\end{array}\right] \text{ and } 
\alpha_{i},w_{i}, k_{i}=\left[\begin{array}{c}
\alpha_{i 1} \\
\alpha_{i 2} \\
\vdots \\
\alpha_{i M}
\end{array}\right],\left[\begin{array}{c}
w_{i 1} \\
w_{i 2} \\
\vdots \\
w_{i M}
\end{array}\right],\left[\begin{array}{c}
k_{i 1} \\
k_{i 2} \\
\vdots \\
k_{i M}
\end{array}\right]_. $$
}
The parameters of $f$ are in turn collected in the vector $\theta = [\theta_1^\top, \theta_2^\top, \dots \theta_N^\top]^\top$.

For each function $f_i$ to be monotonically increasing, which guarantees invertibility, it suffices to ensure that $\alpha_{ij}$ and  $w_{ij}$ are positive for all $j$. The pre-image of $f_i$ is $\mathbb{R}$, but the image is an interval $(\ubar{z}_i, \bar{z}_i)$, which is in accordance with the fact that sensor data are usually restricted to a given dynamic range. If the range is not available a priori but sufficient data is available, bounds for the operation interval can also be easily inferred. 

In order to express how an entry of a time series is predicted from the previous values, let $g_i$ denote the inverse of $f_i$, that is, $y[t] = g(z[t])$ and let $ g:\mathbb{R}^N \rightarrow \mathbb{R}^N$ denote a vector function such that $[\bm g(x)]_i = g_i(x_i)$. Then, \eqref{eq:zfromy} and \eqref{eq:linearvar} imply that
\begin{equation}
\label{eq:zfromz}
z[t] = \bm f\bigg(\sum_{p=1}^{p} A^{(p)}\bm g(z[t-p]) + u[t]\bigg).
\end{equation}
Notice that, in general, there is no closed form for the inverse function $g_i$; however, it is possible to compute it efficiently via a numerical method such as bisection.

The optimization problem for joint learning of $\bm f$ and the VAR parameters is formulated as follows and will subsequently be referred to as \textbf{Formulation A}:

\begin{subequations}
\label{eq:optimization_problem_1}
\begin{align}
    \nonumber
    \min_{\bm f, \{A^{(p)}\}_{p=1}^P} \;\;
    & \sum_{t=P}^T
    \left\| {z}[t]-\bm f\Big(\sum_{p=1}^{p} A^{(p)}\bm g(z[t-p])\Big) \right\|_{2}^{2}
    \\
    & + \lambda \sum_{p=1}^{P}\sum_{n=1}^{N} \sum_{n^{\prime}=1}^{N}\left|a_{n, n^{\prime}}^{(p)}\right| \label{eq:objective_function}
    \\
    \textrm{s. to:}\;\  
    & \sum_{j=1}^{M}\alpha_{i j} 
      = \bar{z}_i 
        - \ubar{z}_i 
        \; \; \forall i
        \label{eq:constraint_alpharange}
    \\
    & b_{i} = \ubar{z}_i
        \; \; \forall i
        \label{eq:constraint_barz}
    \\
    & \alpha_{i j} \geq0 
        \; \; \forall i, j
        \label{eq:constraint_alphapos}
        \\
    & w_{i j} \geq0 
        \;\; \forall i, j
        \label{eq:constraint_wpos_1}
\end{align}
\end{subequations}


The objective function \eqref{eq:objective_function} is a least-squares criterion with a  Lasso regularizer term  over the adjacency coefficients to enforce sparsity in the resulting graph. Here, the hyper-parameter $\lambda$ regulates how sparse the solution is. Notice that this objective function \eqref{eq:objective_function} is non-convex (because it involves composition with $\bm f$ which is non-convex in general) and non-differentiable due to the $l_1$ norm in the Lasso term.
It can be split as $\sum_{t=P}^T C(\left\{A^p\right\}, \theta, t) + q(A^p)$, where
\begin{align}
\label{calculatinga1}
    C\left(A^p, \theta, t\right) =\;\;& \left\| {z}[t]-\bm f\Big(\sum_{p=1}^{p} A^{(p)}\bm g(z[t-p])\Big) \right\|_{2}^{2}
\end{align}
is differentiable, and
\begin{align} 
\label{calculatingq}
   q(A^p) = \lambda\sum_{p=1}^{P} \sum_{n=1}^{N} \sum_{n^{\prime}=1}^{N}\left|a_{n, n^{\prime}}^{(p)}\right|
\end{align}
is not, which motivates the use of proximal algorithms.

The constraints \eqref{eq:constraint_alpharange}, \eqref{eq:constraint_barz} ensure that the image of each  $f$ is in the corresponding sensor dynamic range, and constraints \eqref{eq:constraint_alphapos} and \eqref{eq:constraint_wpos_1}) ensure the invertibility of $f_i$. 
We solve the optimization problem \eqref{eq:optimization_problem_1} stochastically by a technique that combines proximal gradient descent and projected gradient descent. Specifically, the regularization term in the second summand can be  tackled with a proximal parameter update, and the constraints \eqref{eq:constraint_alpharange}, \eqref{eq:constraint_barz}, \eqref{eq:constraint_alphapos} and \eqref{eq:constraint_wpos_1} can be enforced by projection.

The parameter updates are derived as follows. Note that the Lasso penalty only affects the VAR parameters. Thus each $a_{n n^{\prime}}^{(p)}$ is updated iteratively by a proximal update, whereas the parameters $\theta$ are updated by a  gradient step. Letting $t(k)$ denote the time instant used at iteration $k$, we can write the following updates:
\begin{subequations}
\label{eq:calculating_a_and_theta}
\begin{align} 
\label{calculatinga}
 a_{n n^{\prime}}^{(p)(k+1)}&= \operatorname{prox}_{q, \eta}\left({a_{n n^{\prime}}^{(p)(k)}}-\eta \bigg(\frac{d C(A^p, \theta, t(k))}{d {a_{n n^{\prime}}^{(p)}}}\bigg)^\top \right)
 \\
 \label{calculatingtheta}
 \theta_i^{(k+1)} &= \theta_i^{(k)} - \eta \bigg( \frac{d C(A^p, \theta, t(k))}{d \theta_{i}}\bigg)^\top.
\end{align}
\end{subequations}
Note that $q$ in $prox_{q,\eta}$ corresponds to the function defined in \eqref{calculatingq}
and the proximity operator in \eqref{calculatinga} is given by:
\begin{equation}
    \operatorname{prox}_{q, \eta}\left(x\right) = x \left[ 1 - \frac{\eta \lambda}{|x|}\right]_+
\end{equation}
where $ [x]_+ := \max(0,x) $, yielding the well-known soft-thresholding operator \cite{daubechies2004iterative}.

After each parameter update, the NN parameters are projected back onto the feasible set, according to the equation
\begin{subequations}
\label{eq:projection4}
\begin{align}
    \Pi_{S}\left(\theta^{(k)}\right)=\arg \min_{\theta} & \left\|\theta-\theta^{(k)}\right\|_{2}^{2}
\\
\textrm{s. to:}  &
 ~\eqref{eq:constraint_alpharange}, \eqref{eq:constraint_barz}, \eqref{eq:constraint_alphapos}, \eqref{eq:constraint_wpos_1} 
 \end{align}
\end{subequations}
This is a case of projection onto a simplex which is tackled using the projection algorithm in \cite{blondel2014large}.
 
The proximal parameter update requires the computation of the gradient of $C(A^p, \theta, t)$ w.r.t. $ A^p \text{ and } \theta$.
The forward equations can be written as:
\begin{subequations}
\begin{align}
    \label{eq:forward_g}
        \tilde{y}_{i}[t-p] =&g_{i}\left(z_{i}[t-p], \theta_{i}\right)
    \\
    \label{eq:forward_A}
        \textstyle \hat{y}_{i}[t] =&\sum_{p=1}^{p}\sum_{j=1}^{n}  a_{i j}^{(p)} \tilde{y}_{j} [t-p] 
    \\
    \label{eq:forward_f}
        \hat{z}_{i}[t] =&f_{i}\left(\hat{y}_{i}[t], \theta_{i}\right) 
\end{align}
\end{subequations}
where the dependency with the parameter vector $\theta_i$ has been made explicit. The remainder of this section shows the backward equations.

The main challenge to solving this problem is that there is no closed form for the inverse function $g_i$. However, an inverse function can be computed efficiently via bisection as one of the possible methods. On the other hand, automatic differentiation software cannot yield the gradient of $g_i$.
This is circumvented in \cite{luism} using implicit differentiation. To make the paper self-contained, the expressions to compute the gradient of $g_i$ are provided here and the full derivation is shown in Appendix A.
Letting $f_{i}^{\prime}\left(\hat{y}\right)= 
\frac{\partial f_{i}\left(\hat{y}, \theta_{i}\right)}{\partial \hat y}, $   and $S_n = 2(\hat{z}_n[t]-z_n[t])$, the gradient of $C(A^p,\theta,t)$ can be expressed as:

\begin{align*}
\label{derivativecost4}
 \frac{d C(A^p, \theta, t)}{d \theta_{i}}= & S_{i}\bigg(\frac{\partial f_{i}\left(\hat{y}, \theta_{i}\right)}{\partial \theta_{i}}\bigg) +
 \\
 & \sum_{n=1}^{N} S_{n}\bigg(f_{n}^{\prime}(\hat{y}_{n}[t]) \sum_{p=1}^{P} a_{n i}^{(p)} 
 \frac
    {\partial g_{i}\left(z_{i}[t-p],\theta_{i}\right)}
    {\partial \theta_{i}}\bigg)_\cdot
\end{align*}

where

$$\frac{\partial f_i\left(\hat{y}, \theta_{i}\right)}{\partial \theta_{i}} 
= \left[
\frac{\partial f_i\left(\hat{y}, \theta_{i}\right)}{\partial \alpha_{i}} 
\frac{\partial f_i\left(\hat{y}, \theta_{i}\right)}{\partial w_{i}} 
\frac{\partial f_i\left(\hat{y}, \theta_{i}\right)}{\partial k_{i}} 
\frac{\partial f_i\left(\hat{y}, \theta_{i}\right)}{\partial b_{i}} \right]$$ can be obtained by analytic or automatic differentiation. 
The gradient of the inverse function is:
$$
\frac{\partial g_i\left(z, \theta_{i}\right)}{\partial \theta_{i}}=
-\big\{f^{\prime}_i(g_i(z,\theta_{i}))\big\}^{-1}{\bigg(\left.\frac{\partial f_i\left(\tilde{y}, \theta_{i}\right)}{\partial \theta_{i}}\right\vert_{\tilde{y} = g_i\left(z, \theta_{i}\right)}\bigg)}.
$$

Finally, the gradient of $C(A^p, \theta, t)$ w.r.t. the VAR coefficient $a_{n n^{\prime}}^{(p)}$ can be readily calculated as:
\begin{equation} \label{dCwrtthetaa_main}
\frac{d C(A^p, \theta, t)}{d a_{i j}^{(p)}}=S_i f_{i}^{\prime}\left(\hat{y}_{i}[t]\right) \tilde{y}_{j}[t-p]_\cdot
\end{equation}

The detailed derivation of the above expressions is provided in Appendix A. 

The non-convexity of problem \eqref{eq:optimization_problem_1} and the comparatively small number of parameters of the model are factors that increase the risk of falling into low-performance local minima, making the final convergence value of the parameters $\theta$ to be dependent on the initialization. On the other hand, it is expected that the model will accomplish a lower prediction error than the linear VAR model for the same training data. 

A strategy to obtain a non-linear model performing better than the optimal linear one is to initialize $f$ to resemble an identity function at the range of the input data and such that a latent prediction that falls out of the range of typical predictions translates into a measurement prediction that is close to the corresponding extreme (maximum or minimum) value observed in the training data. To this end, it is proposed to initialize $\theta$ such that
\begin{equation}\label{init}
    f_i(\hat{y}_{i}[t], \theta_i) = [\hat{y}_{i}[t]]_{\ubar{z}_i}^{\bar{z}_i}
\end{equation}
approximately holds, where $[\hat{y}_{i}[t]]_{\ubar{z}_i}^{\bar{z}_i} := \max\left(\ubar{z}_i, \min(\hat{y}_{i}[t], \bar{z}_i)\right)$.
Additionally, the latent parameters $\left\{A^p\right\}$ are to be initialized to equal the linear VAR parameters inferred from the training data with a linear VAR estimation method. As a result, the initial (before iterating) prediction error of the initial nonlinear VAR model is equal to that of the linear VAR, and the subsequent iterations (as given in \eqref{eq:calculating_a_and_theta}) will move the parameters in the direction of a solution with a smaller prediction error. Thus, the chances of finding a solution with a lower error than the linear model are increased.

In order to increase the efficiency of the algorithm and avoid an initial training of the non-linearity from a synthetic collection of data points for each of the time series, but only having one pre-trained non-linear function, we derive a set of  transformation equations from the linear model to obtain the desired nonlinearities for their different ranges. A set of transformation equations can be developed by defining a function $\check{f}$ such that  $\check{f}_i(1)$ = $f_i(1) = 1,  $ $\check{f}_i(-1)$ $ = f_i(-1) = -1$, $\check{f}_i(x)$ =  $f_i(x) = x$. Let $\check{\alpha}_i, \check{w}_i, \check{k}_i \text{ and } \check{b}_i $ be the learned parameters corresponding to $\check{f}_i$. The set of transformation equations will be such that $\check{\alpha}_i = c \alpha_i, \check{b}_i =c b_i+d, \check{w}_i = aw_i,\check{k}_i = -w_i B+k_i$ where $c = (\bar{z} - \ubar{z})/2, d = (\bar{z} + \ubar{z})/2, a = -2/(\ubar{z} - \bar{z})$ and $ B = 2\bar{z}/(\ubar{z}-\bar{z})$. The complete derivation of the set of transformation equations is shown in Appendix B. In Sec \ref{Simulation_experiments}, we show experimentally that this initialization speeds up both proposed algorithms.

The steps of the overall method described in this section are summarized in \textbf{Algorithm \ref{alg:NL-VAR_1}}.

\begin{algorithm}
   \caption{Explicit function inversion-based inference}
   \label{alg:NL-VAR_1}
\begin{algorithmic}
   \State {\bfseries Result:} $\boldsymbol{a}_{n, n^{\prime}}^{(p)}$, for $n, n^{\prime}=1, . ., N$ and $p=1,p+1, . ., P$
   \State {\bfseries Input:} data $z_i$, $\lambda$, $N$, order $P$, $M$, $T$, learning rate $\eta$.
   \State {\bfseries Initialize:} $\boldsymbol{a}_{n, n^{\prime}}^{(p)}$, $\theta_{i}$  as stated in \eqref{init}
   \For{$t=P,P+1,...,T$}
      \For{$n=1,2,...,N$} 
      \State
        Generate $ y_n[t]$  from $z_n[t]$ using $g_n$ \eqref{eq:forward_g}
        \State
        Obtain $y_n[t+1]$ using \eqref{eq:forward_A} and 
        \State
        Obtain $z_n[t+1]$ using $f_n$ \eqref{eq:forward_f}
        
        \State
        Network update: $\theta_{n} = \theta_{n}-\eta\frac{d C[t]}{d \theta_{n}}$ \eqref{calculatingtheta}
        
        \State Projection operation \eqref{eq:projection4}
        \For{$n^{\prime}=1,2,...,N$}
            \For{$p=1,2,...,P$}
            \State VAR parameter update: $ a_{n n^{\prime}}^{(p)}[t+1] $ via \eqref{calculatinga}
               
            \EndFor
        \EndFor
      \EndFor
    \EndFor
\end{algorithmic}
\end{algorithm}

\subsection{Latent prediction error minimization-based inference}
\label{section_2_modelling}
As indicated in the previous formulation, the main drawback of the algorithm is associated with the numerical computation of $ \bm g$. Evaluating he function $\bm g$ via bisection adds complexity at each run within the overall algorithm.

Next, we propose an alternative formulation to estimate a nonlinear topology, whose solution leads to a lower-complexity algorithm. The main idea of this formulation is to minimize the prediction MSE in the latent space instead of minimizing it in the measurement space. We will show that minimizing the prediction error in the latent space implies approximately minimizing the prediction error in the measurement space. This is because, as it will become clear later, under certain conditions, the latter is an upper bound of the former. The nonlinearities between measurement and latent space are parameterized here in a way different from that presented in the previous formulation. The function mapping sensor $n$ from latent space to measurement space is now denoted as $r_n$. It has the use of function $f_n$ denoted in the previous section but receives a different symbol as it is parameterized in a different way. The way $r$ is parameterized is via an explicit parameterization of its inverse (denoted by $v$), such that $y[t] = v(z[t])$. The function $v_n$ for sensor $n$ which is the inverse of $r_n$ is parameterized as follows: 

\begin{equation}
\label{l2}
v_{n}(x)=b_{n}+\gamma_{n} x +\sum_{j=1}^{M} \alpha_{n j} h\left(w_{n j}x-k_{n j}\right).
\end{equation}
Note that the way $v_n$ is parameterized is similar to the case of $f_n$ in \eqref{eq:f_model} 
with the addition of the linear term $\gamma_{n} x$, which together with positivity constraints in $\alpha$ and $w$, ensure that the derivative of $v_n$ is at least $\gamma$.


The optimization problem for joint learning of $\bm v$ and the VAR parameters is formulated as follows and will subsequently be referred to as \textbf{Formulation B}:
\begin{subequations}
\label{eq:optimization_problem_b}
\begin{align}
    \min_{{\left\{\left\{{A}_{p}\right\}_{p=1}^{P}, {\theta}\right\}}} \;\;
    &
   \frac{1}{T-P}\sum_{t=P}^{T}\left\|v(z[t])-\sum_{p=1}^{P} A^{(p)} v(z[t-p])\right\|_{2}^{2} \nonumber
   \\
   &+ \lambda \sum_{p=1}^{P}\sum_{n=1}^{N} \sum_{n^{\prime}=1}^{N}\left|a_{n, n^{\prime}}^{(p)}\right| \label{eq:objective_function_b}
    \\
    \textrm{s. to:}\;\  
    &  \alpha_{i j} \geq0 
        \;\; \forall i, j
        \label{eq:constraint_wpos1_b}
    \\
    & w_{i j} \geq0 
        \;\; \forall i, j
        \label{eq:constraint_wpos2_b}
         \\
    & \gamma_{i} \geq 0
        \;\; \forall i 
        \label{eq:constraint_wpos3_b}
            \\
    & \frac{\sum_{t=0}^{T-1} v_{i}(z_i[t])}{T} = 0
        \;\; \forall i
        \label{eq:constraint_wpos4_b}
         \\
    &  \frac{\sum_{t=0}^{T-1}(v_{i}(z_i[t]))^{2}}{T-1} = 1
        \;\; \forall i
        \label{eq:constraint_wpos5_b}
\end{align}
\end{subequations}

As it can be seen in the problem formulation, the prediction error is minimized in the latent space. This is justified because, as Theorem 1 shows next, the prediction MSE in the latent space is also an upper bound of the prediction MSE in the measurement space when the set of functions $\{r_c(x)\}$ are Lipschitz continuous. Therefore, minimizing in the latent space entails approximately minimizing in the measurement space.

With $\hat{z}$ and $\hat{y}$ denoting the prediction in measurement and latent space respectively, we state the following theorem.

\begin{theorem}
if $\bm r_n()$ is $L_{r_n}$-Lipschitz continuous and $z_n[t]$ and $y_n[t]$ are related as $z_n[t] = r_n(y_n[t])$, then the following bound holds:
\begin{align}
\label{eq:theorem1_bound}
\sum_{n=1}^{N}\mathbb{E}\left[\left\|\hat{z}_{n}[t]-z_{n}[t]\right\|_{2}^2\right] \nonumber
\\
\leq\left(\max _{n} L_{r_{n}}\right)^2 
&
\sum_{n=1}^{N} \mathbb{E}
\left[
 \left\|\hat{y}_{n}[t]-y_{n}[t]\right\|_{2}^2
\right]
\end{align}
\end{theorem}
\begin{proof}[Proof]

Given that $L_{r_n}$ is Liptschitz continuous with Lipschitz constant $L_{r_n}$, the following holds:

\begin{align}
\label{eq:l_r}
&\|r_n(\hat{y})-r_n(y)\|_{2} \leq L_{r_n}\|\hat{y_n}-y_n\|_{2}
\\
\label{eq:l_yz}
&\|\hat{z_n}[t]-z_n[t]\|_{2} \leq L_{r_n}\|\hat{y_n}[t]-y_n[t]\|_{2}
\end{align}

Squaring \eqref{eq:l_yz} equation and taking expectation, we obtain the following:

\begin{align}
\label{eq:l_least_squares}
\sum_{n=1}^{N}\mathbb{E}\left[\left\|\hat{z}_{n}[t]- z_{n}[t]\right\|_{2}^2\right]  
&\leq \sum_{n=1}^{N} (L_{r_{n}})^2 \, \mathbb{E}\left[\left\|\hat{z}_{n}[t]-z_{n}[t]\right\|_{2}^2\right] \nonumber
\\
\leq\left(\max _{n} L_{r_{n}}\right)^2 &\sum_{n=1}^{N} \mathbb{E}\left[\left\|\hat{y}_{n}[t]-y_{n}[t]\right\|_{2}^2\right]
\end{align}

\end{proof}

The Lipschitz continuity constant of a specific instance of function $v$ can be obtained from a differential property as

\begin{equation}
\label{ld}
L_{r_{n}}=1 / {\min _{x^{\prime}}\left\{\frac{d v_{n}(x)}{d x} \mid_{x^{\prime}=x}\right\}}
\end{equation}

Intuitively, if $v_n$ is too flat, then $r_n$ is too steep, which implies that a small variation in the prediction in the latent space can be associated with a large variation in the prediction in the measurement space, which can entail a larger prediction MSE in the measurement space as the bound becomes loose.

Now that the rationale for having objective function \eqref{eq:objective_function_b} is clear, we explain the constraints:
\eqref{eq:constraint_wpos4_b} and \eqref{eq:constraint_wpos5_b} ensures that the mean of the output of $v_n$ is 0 and the variance of the output of $v_n$ is 1 inside the latent space. The idea to enforce these constraints is to have $v_n$ in the proper dynamic range so that it is not flat. It enacts a nonlinear normalization into the latent space. Notice that if $v_n$ is flat, the left-hand side of \eqref{ld} goes to infinity, making $r_n$ not Lipschitz continuous anymore. Constraints \eqref{eq:constraint_wpos1_b}, \eqref{eq:constraint_wpos2_b} and \eqref{eq:constraint_wpos3_b} ensures that each function $v_n$ is invertible.

Similarly to the first formulation, we also enforce sparsity-inducing penalties for the VAR coefficients, and the regularization term in the second summand is again tackled by using a proximal parameter update. 

Notice that, as opposed to the first formulation, the optimization problem does not explicitly include the inverse function, and hence the burden of computing the inverse function with the bisection method is avoided resulting in reduced complexity.

Next, we aim to solve the optimization problem \eqref{eq:optimization_problem_b} using Lagrangian duality. More specifically, we dualize constraints \eqref{eq:constraint_wpos4_b} and \eqref{eq:constraint_wpos5_b}. The remaining constraints can be easily enforced by using a projection operation.
The objective function and the constraints \eqref{eq:constraint_wpos5_b} and \eqref{eq:constraint_wpos4_b} are of non-convex nature. Notice that since the optimization problem is not convex, we cannot theoretically claim that an iterative algorithm based on duality will achieve a globally optimal solution satisfying all the constraints. However, as we will show in the experimental results section, our algorithm achieves satisfactory results.

With $\beta$ and $\mu$ respectively denoting the dual variables associated with constraint \eqref{eq:constraint_wpos4_b} and \eqref{eq:constraint_wpos5_b} of the optimization problem \eqref{eq:objective_function_b}, the partial Lagrangian based on \eqref{eq:objective_function_b} can be written as:
\begin{align}
\label{eq:lagragian}
\mathcal{L}
\left(
    \left\{{A}_{p}\right\}_{p=1}^{P}, \theta,\beta,\mu
\right) 
= &f_o\left(
\left\{{A}_{p}\right\}_{p=1}^{P}, \theta
\right)
\\
&+ \beta^\top g_1(\theta) + \mu^\top  g_2(\theta) \nonumber
\end{align}
where
\begin{align}
\label{eq:f_o}
&f_o{\left(\left\{{A}_{p}\right\}_{p=1}^{P}, {\theta}\right)} = \nonumber
\\
&\frac{1}{T-P}\sum_{t=P}^{T-1}\left\|v(z[t])-\sum_{p=1}^{P} A^{(p)} v(z[t-p])\right\|_{2}^{2}\nonumber
\\
&+ \lambda \sum_{p=1}^{P}\sum_{n=1}^{N} \sum_{n^{\prime}=1}^{N}\left|a_{n, n^{\prime}}^{(p)}\right|
\end{align}
\begin{equation}
\label{eq:g_1}
[g_1\{{\theta}\}]_i = \frac{\sum_{t=0}^{T-1} v_{i}(z_i[t])}{T}, \;\; \forall i 
\end{equation}
\begin{equation}
\label{eq:g_2}
[g_2\{{\theta}\}]_i = \frac{\sum_{t=0}^{T-1}(v_{i}(z_i[t]))^{2}}{T-1}-1, \;\; \forall i 
\end{equation}

The following steps show how the optimization problem can be solved using the stochastic primal-dual algorithm \cite{boyd_dual}. Considering $\eta_p$ and $\eta_d$ as primal and dual learning rate, The following steps are derived:

Let us define a stochastic version of the partial Lagrangian function:

\begin{align}
\label{eq:stochastic_lagragian}
\tilde{\mathcal{L}}{\left(\left\{{A}_{p}\right\}_{p=1}^{P}, {\theta},\beta,\mu; t\right)} 
&= \tilde{f_o}\left(\left\{{A}_{p}\right\}_{p=1}^{P}, {\theta}; t\right)+ \beta^\top  \tilde{g_1}\{{\theta}; t\} \nonumber
\\
& + \mu^\top  \tilde{g_2}\{{\theta}; t\}
\end{align}

In the next paragraphs $\tilde{\mathcal{L}},\tilde{f_o},\tilde{g_1} \text{ and } \tilde{g_2}$ are defined such that 
\begin{equation}
\mathcal{L}{\left(\left\{{A}_{p}\right\}_{p=1}^{P}, {\theta},\beta,\mu\right)} = \sum_{t=0}^{T-1} \tilde{\mathcal{L}}{\left(\left\{{A}_{p}\right\}_{p=1}^{P}, {\theta},\beta,\mu; t\right)}.
\end{equation}

\text { Accordingly, the stochastic contribution to $f_o$ is defined as:} 

\begin{equation}
\begin{aligned}
&\tilde{f_o} 
\left(\left\{{A}_{p}\right\}_{p=1}^{P}, {\theta}\right) 
\\
&=\left\{
\begin{array}{lr}
0,  & \hspace{-8mm}{0<t<P}\\
\frac{1}{T-P}\Big[\left\|v(z[t])-\sum_{p=1}^{P} A^{(p)} v(z[t-p])\right\|_{2}^{2} & \nonumber
\\
+\lambda\sum_{p=1}^{P}\sum_{n=1}^{N} \sum_{n^{\prime}=1}^{N}\left|a_{n, n^{\prime}}^{(p)}\right|\Big], & { t \geq P}
\end{array}\right.
\end{aligned}
\end{equation}

then, we have that: 
\begin{equation}
f_o\left(\left\{{A}_{p}\right\}_{p=1}^{P}, {\theta}\right) = \sum_{t=0}^{T-1} {\tilde{f_o}\left(\left\{{A}_{p}\right\}_{p=1}^{P}, {\theta}; t\right)}.
\end{equation}

\begin{equation}
\text { similarly, consider: } {[\tilde{g_1}\{{\theta}; t\}]_i}
 = \frac{q_{i}(z_i[t])}{T}, \;\; \forall i .
\end{equation}

\begin{equation}
[g_1\{{\theta}\}]_i
 = \sum_{t=0}^{T-1} {[\tilde{g_1}\{{\theta}; t\}]_i}  \;\; \forall i.
\end{equation}

\begin{equation}
\text { Also } [\tilde{g_2}\{{\theta}; t\}]_i
 =  \frac{(v_{i}(z_i[t]))^{2}-(T-1)/T)}{T-1} \;\; \forall i.
\end{equation}
 
\begin{equation}
[g_2\{{\theta}\}]_i
 = \sum_{t=0}^{T-1} {[\tilde{g_2}\{{\theta}; t\}]_i}  \;\; \forall i.
\end{equation}

 Let $t(k)$ denote the time instant used at iteration $k$, the stochastic primal update equations are:
\begin{equation}
\label{eq:theta_update_stoch}
{\theta}_i[k+1] = {\theta}_i[k] - \eta_p \frac{\partial \mathcal{\tilde L}{\left(\left\{{A}_{p}[k]\right\}_{p=1}^{P}, {\theta}[k],\beta[k],\mu[k];t(k)\right)}}{\partial{\theta}_{i}[k]}
\end{equation}


\begin{align}
\label{eq:a_prox_lip}
a_{n n^{\prime}}^{(p)}[k+1]&= \operatorname{prox}_{q, \eta_p}\Bigg({a_{n n^{\prime}}^{(p)}(k)} \nonumber
\\
-&\eta_p 
\frac{\partial\mathcal{ \tilde L}{\left(\left\{{A}_{p}[k]\right\}_{p=1}^{P}, {\theta}[k],\beta[k],\mu[k];t(k)\right)}}{\partial a_{n, n^{\prime}}^{(p)}[k]}
\Bigg)
\end{align}


Similarly, the stochastic dual update equations are:
\begin{align}
\label{eq:beta_update_stocastic}
&\beta_i[k+1]  = \beta_i[k] \nonumber
\\
&+ \eta_d \frac{\partial \mathcal{\tilde L}{\left(\left\{{A}_{p}[k+1]\right\}_{p=1}^{P}, {\theta}[k+1],\beta[k],\mu[k];t(k)\right)}}{\partial \beta_i[k]}
\\
& = \beta_i[k] + \eta_d [\tilde{g_1}\{{\theta[k+1];t(k)}\}]_i
\end{align}
\begin{align}\label{eq:mu_update_stocastic}
&\mu_i[k+1]  = \mu_i[k] \nonumber
\\
&+ \eta_d \frac{\partial \mathcal{\tilde L}{\left(\left\{{A}_{p}[k+1]\right\}_{p=1}^{P}, {\theta}[k+1],\beta[k],\mu[k];t(k)\right)}}{\partial \mu_i[k]}
\\
& = \mu_i[k] + \eta_d [\tilde {g_2}\{{\theta[k+1];t(k)}\}]_i
\end{align}

As discussed in Sec. \ref{section_2_modelling}, a strategy to increase the chance of obtaining a non-linear model performing better than the linear one is to initialize the nonlinearity to resemble an identity function at the range of the input data. The initial form of the function $v_i$ is required to resemble as much as possible the inverse of the initial shape of the function $f$ used in Formulation A. Since the initial $f$ in formulation A behaves like the identity in the range of the input data and is flat out of that range, the initial $v_i$ in Formulation B is sought to behave like the identity in the range of the input data and have a steep slope out of that range. Following steps similar to those described for the initialization of $f$ in Sec. \ref{section_1_modelling} the parameters for each node can be obtained by transforming the parameters obtained from training a standard initial function which behaves as an identity between -1 and 1.

The steps described in this section are summarized in \textbf{Algorithm 2}.

\begin{algorithm}
   \caption{Latent error minimization-based inference}
   \label{alg:NL-VAR_2}
\begin{algorithmic}
   \State {\bfseries Result:} $\boldsymbol{a}_{n, n^{\prime}}^{(p)}$, for $n, n^{\prime}=1, . ., N$ and $p=1,p+1, . ., P$
   \State {\bfseries Input:} data $z_i$, $\lambda$, $N$, order $P$, $M$, $T$, learning rates $\eta_p, \eta_d$.
   \State {\bfseries Initialize:} $\boldsymbol{a}_{n, n^{\prime}}^{(p)}$,$ \theta_{i}$ 
   \For{$t=P,P+1,...,T$}
      \For{$n=1,2,...,N$} 
      \State
        Generate $ y_n[t]$  from $z_n[t]$ using $v_n$
        \State
        Obtain $y_n[t+1]$ via \eqref{eq:var_scalar} 
        \State
        Obtain $\tilde{\mathcal{L}}{\left(\left\{{A}_{p}\right\}_{p=1}^{P}, {\theta},\beta,\mu; t\right)} $ via \eqref{eq:lagragian} 
        \State
        Network parameter update: $\theta_{n}$ via \eqref{eq:theta_update_stoch}
        \State Dual parameters update: $\beta$, $\mu$ via \eqref{eq:beta_update_stocastic}, \eqref{eq:mu_update_stocastic}
        \State Projection operation \eqref{eq:projection4}
        \For{$n^{\prime}=1,2,...,N$}
            \For{$p=1,2,...,P$}
            \State VAR parameter update: $ a_{n n^{\prime}}^{(p)}[t+1] $ via \eqref{eq:a_prox_lip}
            \EndFor
        \EndFor
      \EndFor
    \EndFor
\end{algorithmic}
\end{algorithm}

%% file: IEEE_Journal_TSP/experiments.tex
\section{Simulation Experiments}
\label{Simulation_experiments}
In this section, we conduct comprehensive numerical tests to assess the performance of our algorithms formulation A (f\textunderscore A) and formulation B (f\textunderscore B) on synthetic and real data sets. We provide comparisons against the best four current competitors: cMLP (component-wise Multi-Layer Perceptrons), cLSTM (component-wise Long Short-Term Memory), cRNN (component-wise Recurrent Neural Networks)\cite{tank2021neural}, and linear VAR. 

The proposed algorithms are evaluated based on the performance metrics described next, where expectations are approximated by the Monte Carlo method.

The probability of false alarm ($P_{\textrm{FA}}$) and probability of detection ($P_{\textrm{D}}$) are used to numerically compare the edge-identification performance of the algorithms. The $P_{\textrm{FA}}$ is the likelihood that the algorithm detects the existence of a dependence that does not exist, whereas the $P_{\textrm{D}}$ is the likelihood that the algorithm discovers a dependence that is really existent in the network. In our experiments, we suppose that there is a detectable edge from the 
$p^{th}$ time-lagged value of the 
$n^{th}$ sensor to  $n^{th}$ sensor if the absolute value of coefficient ${a}_{n, n^{\prime}}^{(p)}$ is greater than a prespecified threshold $\delta$. Letting ${\hat{a}}_{n, n^{\prime}}^{(p)}$ be a binary variable that indicates that $a_{n, n^{\prime}}^{(p)}$ is detected as nonzero, it is computed as ${\hat{a}}_{n, n^{\prime}}^{(p)}=\mathbbm{1}\left\{| a_{n,n^{\prime}}^{(p)} |>\delta\right\}$, where $\mathbbm{1}\{x\}$ denotes the indicator function, taking value 1 when $x$ is true and 0 when $x$ is false. With ${a}_{n, n^{\prime}}$ denoting the presence of a true edge, $P_{\textrm{FA}} $ and $P_{\textrm{D}} $ are defined as 
\begin{align}
\label{eq:P_D}
P_{\mathrm{D}} \triangleq 1-\frac{\sum_{n \neq n^{\prime}} \sum_{p=1}^P \mathbb{E}\left[\mathbbm{1}\left\{|a_{n, n^{\prime}}^{(p)}|>\delta\right\} \mathbbm{1}\left\{a_{n, n^{\prime}}=1\right\}\right]}{\sum_{n \neq n^{\prime}} \sum_{p=1}^P \mathbb{E}\left[\mathbbm{1}\left\{a_{n, n^{\prime}}=1\right\}\right]}
\end{align}

\begin{align}
\label{eq:P_FA}
P_{\mathrm{FA}} \triangleq \frac{\sum_{n \neq n^{\prime}} \sum_{p=1}^P \mathbb{E}\left[\mathbbm{1}\left\{|a_{n, n^{\prime}}^{(p)}|>\delta\right\} \mathbbm{1}\left\{a_{n, n^{\prime}}=0\right\}\right]}{\sum_{n \neq n^{\prime}} \sum_{p=1}^P \mathbb{E}\left[\mathbbm{1}\left\{a_{n, n^{\prime}}=0\right\}\right]}
\end{align}

With an increase in $\delta$, both $P_{\textrm{D}}$ and $P_{\mathrm{FA}}$ decrease, eventually reaching zero.

In our study, we measure the prediction accuracy using normalized mean squared error (NMSE):

\begin{align}
\label{eq:NMSE_Eq}
\operatorname{NMSE}(\mathrm{T})=\frac{\sum_{n=1}^N\sum_{t=1}^T\left(y_n\left[t\right]-\hat{y}_n\left[t\right]\right)^2}{\sum_{n=1}^N\sum_{t=1}^T\left(y_n\left[t\right]\right)^2}
\end{align}
where $\hat{y}_n[t]$ is the estimate of the time series generated by the $n^{th}$ node at time instant $t$.
The captions and legends of the figures provide a list of all the experimental parameter values.

\begin{figure}[ht]
\centering
\includegraphics[width = 0.8\columnwidth]{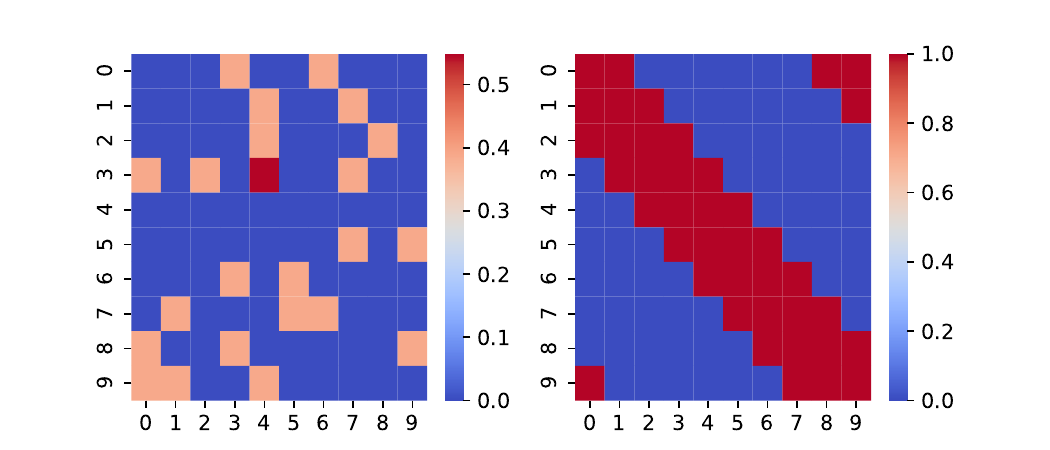}
\caption{True causal dependencies VAR model with $P=4$ (left) and Lorentz $F = 10$ (right)
} 
\label{fig:true}
\end{figure}

\subsection{Experiments with synthetic data}
\label{sec:synthetic}
\begin{table*}[ht]
\caption{Comparison of AUROC for VAR causality selection among different approaches, as a function of the VAR lag order and the length of the time series T . Averaged over 10 experimental runs}
\label{table_1}
\centering
\begin{tabular}{p{0.1\linewidth}p{0.1\linewidth}p{0.1\linewidth}p{0.1\linewidth}p{0.1\linewidth}p{0.1\linewidth}p{0.1\linewidth}}
\hline
Model &   & VAR lag order (P) = 4  &   &   & VAR lag order (P) = 8  &  \\
\hline
T & T = 250 & T = 500  & T = 1000  & T = 250  & T = 500  & T = 1000\\
\hline
formulation A & 0.7562 & 0.9299  & 0.9796  & 0.6437  & 0.6833  & 0.7379\\
Linear VAR & 0.8159 & 0.9153  & 0.9645  & 0.6685 & 0.6726  & 0.7202\\
formulation B & 0.7795 & 0.9435  & 0.9976  & 0.6137  & 0.6557 & 0.8084\\
cMLP & 0.6390 & 0.7424  & 0.7522  & 0.5551  & 0.5736  & 0.5845\\
cRNN & 0.6519 & 0.7947  & 0.8922  & 0.5672  & 0.5827  & 0.5935\\
cLSTM & 0.5505 & 0.5837  & 0.6116  & 0.5350  & 0.5716  & 0.5833\\
\hline
\end{tabular}
\end{table*}

\begin{table*}[ht]
\caption{comparison of AUROC for VAR causality selection among different approaches, as a function of the force constant F and the length of the time series T . Averaged over 10 experimental runs.}
\label{table_2}
\centering
\begin{tabular}{p{0.1\linewidth}p{0.1\linewidth}p{0.1\linewidth}p{0.1\linewidth}p{0.1\linewidth}p{0.1\linewidth}p{0.1\linewidth}}
\hline
Model &   & F = 10  &   &   & F = 40  &  \\

\hline
T & T = 250 & T = 500  & T = 1000  & T = 250  & T = 500  & T = 1000\\
\hline
formulation A & 0.8833 & 0.9636  & 0.9727  & 0.7141  & 0.7795  & 0.7843\\
Linear VAR           & 0.9372 & 0.9627  & 0.9716  & 0.7237  & 0.7712  & 0.7776\\
formulation B & 0.9158 & 0.9684  & 0.9785  & 0.7202  & 0.7855  & 0.8081\\
cMLP & 0.9801 & 0.9734  & 0.9827  & 0.9425  & 0.9811  & 0.9808\\
cRNN & 0.9002 & 0.9999  & 1.0000  & 0.9265  & 0.9613  & 0.9915\\
cLSTM & 0.9752 & 0.9874  & 0.9894  & 0.7732  & 0.7505  & 0.8125\\
\hline
\end{tabular}
\end{table*}

\begin{figure*}[ht]
\centering
\captionsetup{justification=centering}
\includegraphics[width = 1.2\columnwidth]{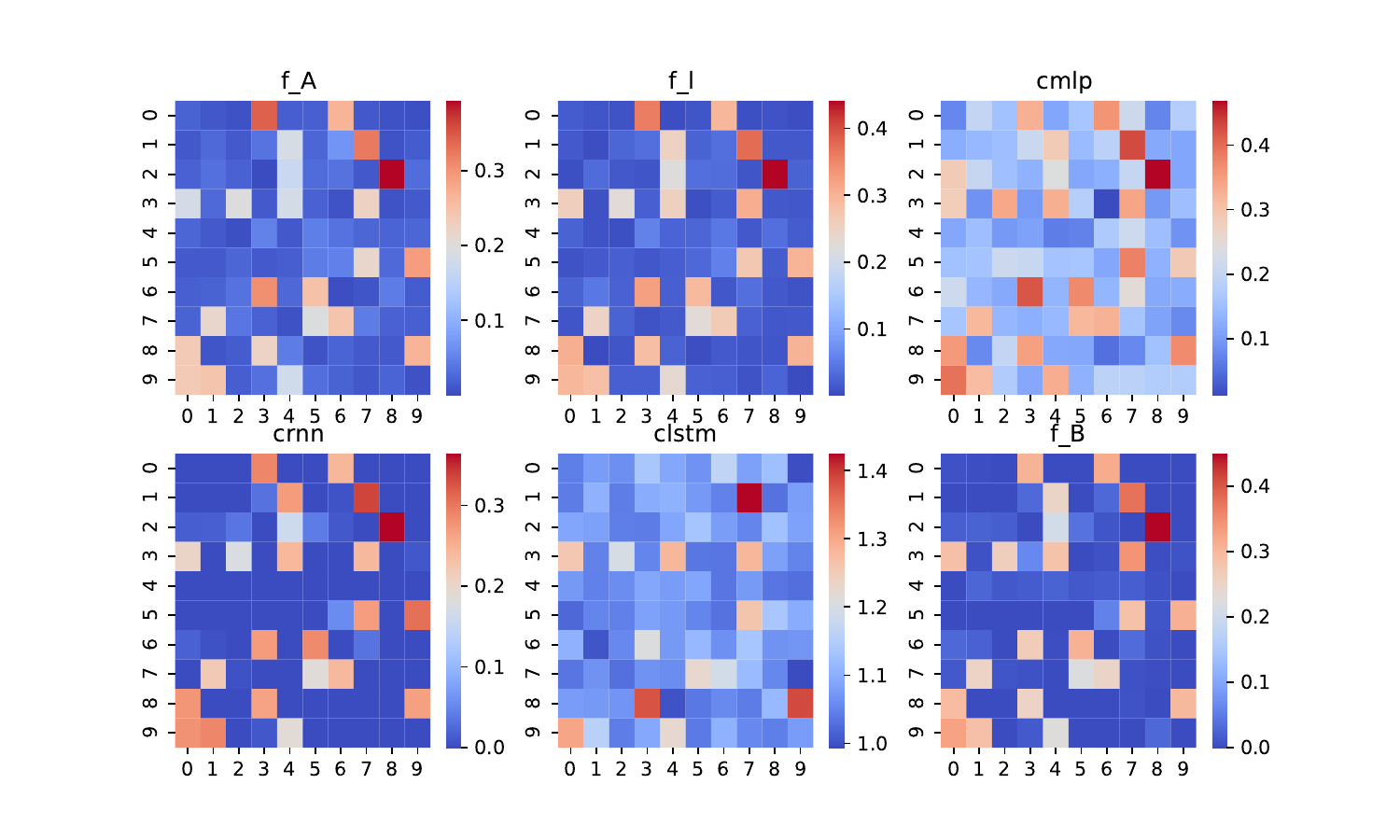}
\caption{Learned causal dependencies from the data generated from VAR model with $P=2$ and $T = 1000$
} 
\label{fig:VAR}
\end{figure*}

\begin{figure*}[ht]
\centering
\captionsetup{justification=centering}
\includegraphics[width = 1.2\columnwidth]{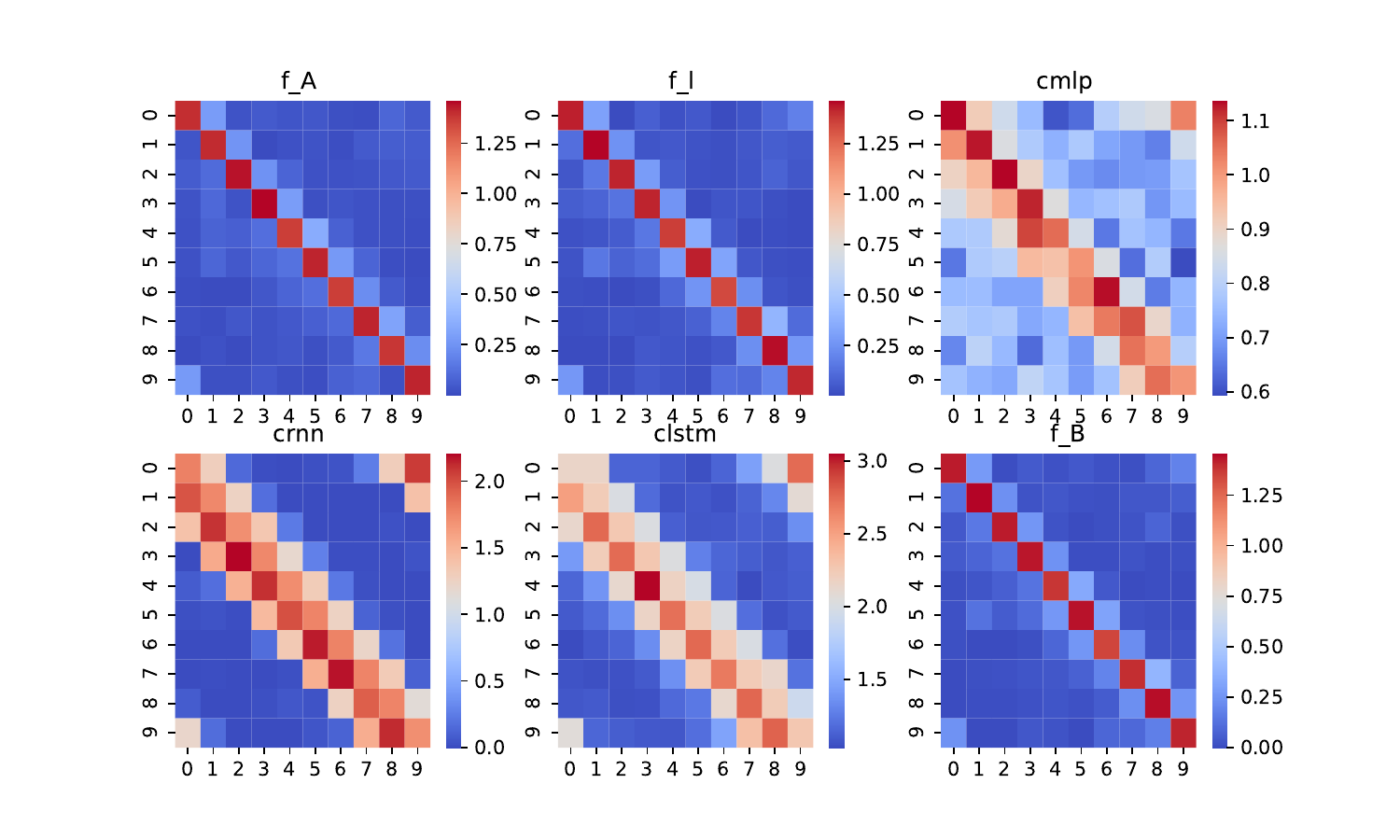}
\caption{Learned causal dependencies from the data generated from Lorentz model with $F = 10 $ and $T = 1000$
} 
\label{fig:lorentz}
\end{figure*}

We use both formulated algorithms to find VAR-based dependency networks in simulated data from a nonlinear VAR model matching the assumption in Sec. \ref{modelling_assumption} and from a Lorenz-96 process \cite{lorentz96}, a nonlinear model of climate dynamics, to compare and analyze the performance of our approaches with cMLP, cRNN, and cLSTM. Overall, the findings demonstrate that the proposed approaches can rebuild the underlying nonlinear VAR structure. The VAR experiment findings are presented first, followed by the Lorentz results. Note that we used Hidden units $H =10 $ for formulations A, B and $H = 100$ for cMLP, cRNN, and cLSTM throughout the experiments. 

The sparsity hyper-parameters $\lambda$ for different algorithms are selected via grid search based on the held-out validation error (note that the optimal $\lambda$ for different methods are not necessarily equal under different conditions). 

The final adjacency matrices are computed by taking the $l_2$ norm (Euclidean norm) along the third dimension (axis 3) of the estimated three-dimensional tensor $\left\{A^p\right\}$.

The metric used to compare the different approaches is the area under the receiver operating characteristic (AUROC). The ROC curve is traced by selecting different values of threshold $\delta$ and for each of these values a point ($P_{\textrm{FA}}$, $P_{\textrm{D}}$) is computed from 10 Monte Carlo runs. The reported AUROC is the area under the linear interpolant joining the aforementioned points. A topology identification algorithm with a high AUROC value generally achieves operation points with high $P_{\textrm{D}}$ and low $P_{\textrm{FA}}$, indicating that it can accurately identify network topologies while minimizing the occurrence of false positives.

\begin{figure}[ht]
\centering
\includegraphics[width = 0.9\columnwidth]{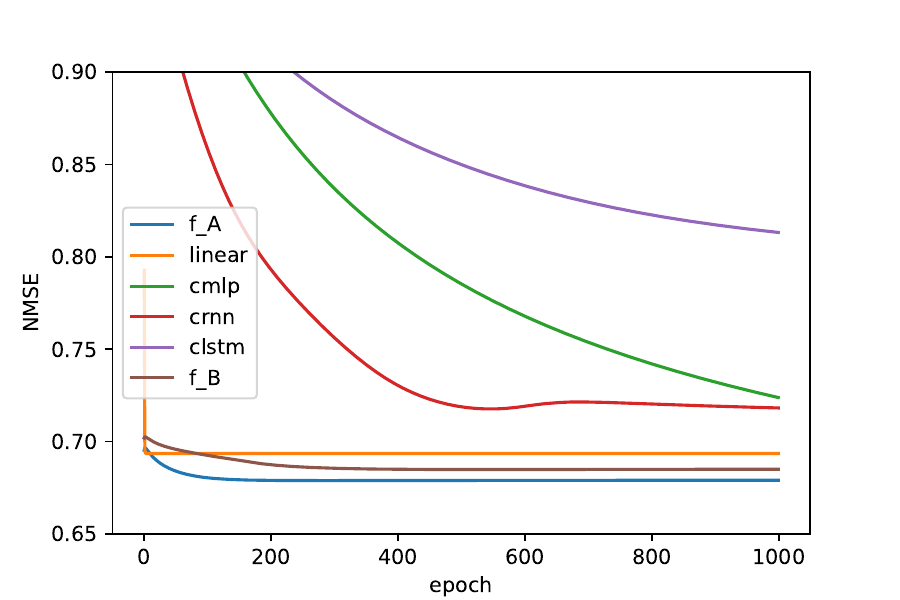}
\caption{NMSE comparison of formulation A, B, cMLP, cRNN, cLSTM, and VAR from  data generated through nonlinear VAR model with lag order $P = 2 $ and $T = 1000$
} 
\label{fig:nmse_var}
\end{figure}

\begin{figure}[ht]
\centering
\includegraphics[width = 0.9\columnwidth]{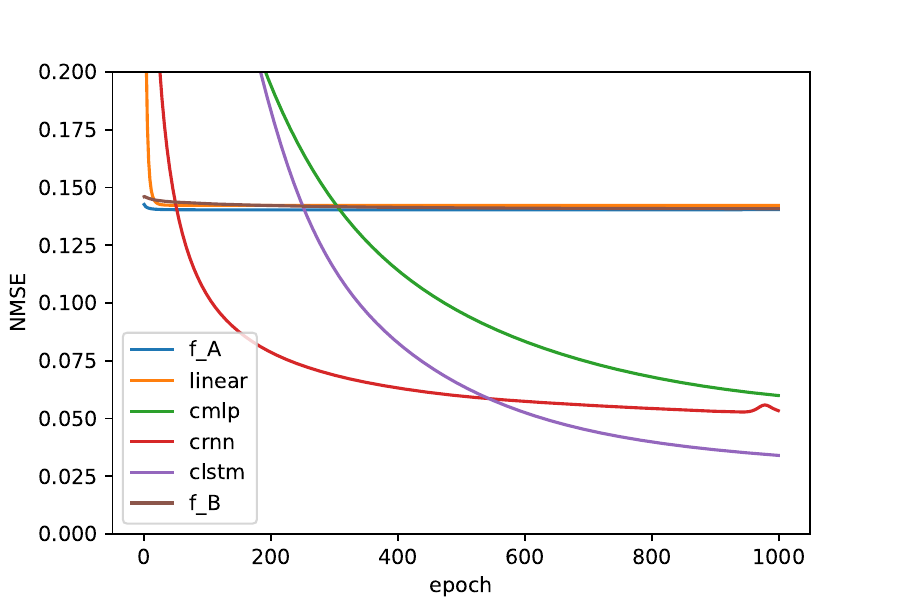}
\caption{NMSE comparison of formulation A, B, cMLP, cRNN, cLSTM, and VAR from  data generated  from Lorentz model with $F = 10 $ and $T = 1000$
} 
\label{fig:nmse_lorentz}
\end{figure}

The following subsections describe how the synthetic data are generated. Along all experiments, each generated dataset is split into training (70\%), validation (20\%), and test (10\%) subsets.

\subsubsection{Nonlinear VAR Model}
\label{sec:synthetic_data_nonlinear_VAR}
We generate graph-connected time series based on the nonlinear VAR (NL-VAR) model. The parameter values are $N = 10$, $T =10000$, $P =4$, and $P = 8$. When generating NL-VAR data set for $P = 4$ and $8$, we set the lag order parameter to $4$ and $8$ respectively. The VAR parameters ${a_{n n^{\prime}}^{(p)}}$ are drawn from a Bernoulli distribution with (edge) probability $p_e = 0.15$. In order to make the underlying VAR process stable, we re-scale the generated coefficient matrix \ref{fig:true}.The nonlinearity $f_i(\cdot)$ (a monotonically increasing nonlinear function) is randomly generated by drawing random values for the parameters $\theta$ from a uniform distribution and then applying the model in equation \eqref{eq:f_model}.

The nonlinear model is initialized  following the heuristic steps described at the end of Sec. \ref{section_1_modelling}.
Results are displayed in Table \ref{table_1}. The AUROC for the proposed formulations A and B, linear VAR, cMLP, cRNN, and cLSTM approaches for three values of the time series length, $T \in \{250, 500, 1000\}$ with lag order $P \in \{4,8\} $ is calculated. The performance of all  models improves at larger T for both lag orders ($P=4$ and $P = 8$). Formulations A and B outperform the linear model (VAR) for a large enough value of T. This result is expected as the model has a slightly larger expressive power, requiring a moderate increase in T to not overfit. Formulations A, B, and VAR outperform state-of-the-art cMLP, cRNN, and cLSTM models. The performance of other models seems to deteriorate over a higher lag value. It is clear from Fig. \ref{fig:VAR} and Fig. \ref{fig:true} that the estimates ($a_{n n^{\prime}}^{(p)}$) of formulation B are very close to the ground truth, and they outperform the other algorithms for $P = 2$ and $ T =1000$. From Fig. \ref{fig:nmse_var}, the results seem to suggest that the prediction capability for formulations A and B is better than that of cMLP, cRNN, and cLSTM. 

\subsubsection{Lorentz Model}
In an N-dimensional Lorenz model, the continuous dynamics are given by
\begin{align}
\label{eq:lorentz_equation}
\frac{dx_{ti}}{dt} = (x_{t(i+1)} - x_{t{(i-2)}})x_{t(i-1)} - x_{ti} +F,
\end{align}
where $x_{t(-1)} = x_{t(p-1)}, x_{t0} = x_{tp},x_{t(p+1)} = x_{t1}$; higher values of the force constant $F$ entail a stronger nonlinearity and more chaotic behavior in the time series. The data time series generated in this case corresponds to a discrete-time simulation of a multivariate Lorentz-96 model with $N=10$ series where the nonlinear dependencies follow a sparse pattern as depicted on the right pane of Fig. \ref{fig:true}.

AUROC values were calculated for formulations A, B, linear VAR, cMLP, cRNN, and cLSTM across time series lengths $T=250$, $T=500$, and $T=1000$, for force constant $F$ taking values 10 and 40. According to Table \ref{table_2}, for $F=10$, all models for $T >500$ has obtained AUROC $> 0.95 $. For more chaotic series with $F=40$, cMLP and cRNN kept a performance above 0.95, and cLSTM and interpretable models attained an AUROC value between $0.7$ and $ 0.8 $. The simplifying modeling offers interpretability with a slight loss in expressive power. In highly chaotic time series ($F=40$), performance moderately declines but remains competitive with DL models for less chaotic processes ($F=10$). AUROC improves with larger $T$, with f\_A and f\_B outperforming linear VAR for $T>500$. The cRNN model estimates closely match the ground truth, especially for $F=10$. 
Fig. \ref{fig:nmse_lorentz} shows that the train NMSE for formulations A and B is better than that of linear VAR by a small margin, whereas the DL models perform significantly better at prediction. This result contrasts with the high and similar AUROC values shown in Table \ref{table_2}, and suggests that the proposed modeling assumption cannot capture the complexity of the Lorentz model.

\subsection{Experiments with real data sets}

\begin{figure}[ht]
\centering
\includegraphics[width = 0.9\columnwidth]{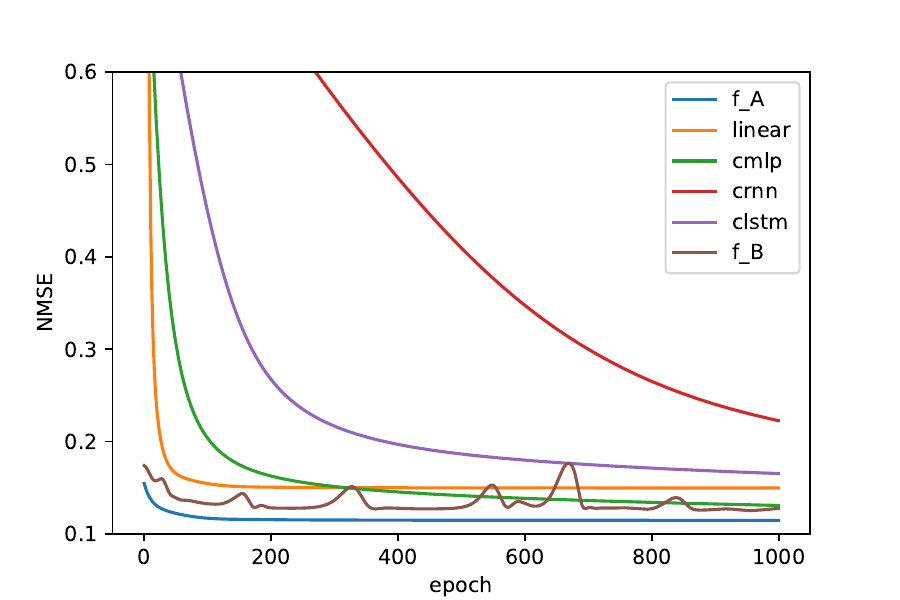}
\caption{NMSE comparison of formulation A, B, cMLP, cRNN, cLSTM, and VAR using real data from Lundin separation facility. $N = 24 $ and $T = 4000$.
} 
\label{fig:nmse_lundin}
\end{figure}

\begin{figure}[ht]
\centering
\includegraphics[width = 0.8\columnwidth]{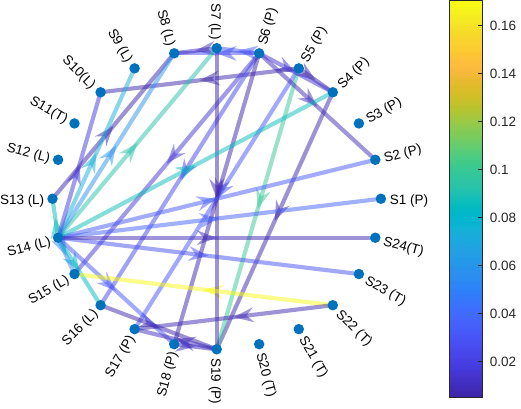}
\caption{Causal dependencies estimated using formulation B for real data from Lundin separation facility with $N  = 24 $ and $T = 4000$
} 
\label{fig:adj_lundin}
\end{figure}

In this section, we conduct experiments using data collected from a sensor network at the Edvard Grieg offshore oil and gas platform. We have 24 time series, each representing sensor readings from decantation tanks measuring temperature (T), pressure (P), or oil level (L). Our goal is to uncover hidden dependencies and predict the system's short-term future state in terms of variables such as pressure and temperature, which may be influenced by physical tank proximity, pipeline flows, and control mechanisms. To create these time series, we uniformly sample sensor values every 5 seconds, resulting in 4000 samples in total.

We employ various methods, including Formulations A, B, cMLP, cRNN, cLSTM, and linear VAR, to infer these variable relationships. The optimal $\lambda$ is determined through a grid search and cross-validation process. Using the parameters learned from Formulation B, we construct an adjacency matrix by computing the $l_2$ norm of the parameter vector for each pair of nodes. The resulting graph is visualized in Fig. \ref{fig:adj_lundin}, where self-loops are removed, and arrow colors indicate edge weights.

Additionally, Fig. \ref{fig:nmse_lundin} displays the performance of all methods in terms of training NMSE. Formulations A and B consistently outperform VAR, cMLP, cRNN, and cLSTM, with Formulation A achieving the lowest prediction NMSE. This aligns with our results from synthetic nonlinear VAR data in Sec. \ref{sec:synthetic_data_nonlinear_VAR}, where Formulation B demonstrated superior topology identification performance. Since there is no ground truth available for the topology in this case, we visualize the graph identified by Formulation B for clarity.

%% file: IEEE_Journal_TSP/conclusion.tex
\section{Conclusion}

To discover the dependencies that are inherent to  a nonlinear multivariate model, a modelling technique has been described, formulated and validated. The main modelling idea is that a nonlinear VAR model can be expressed as the composition of a linear VAR model and a set of univariate, invertible nonlinear functions. A NN is associated with each variable in such a model to express the non-linear relation between a real-world sensor and a latent variable that is part of a VAR model that can be directly associated with a graph. In order to increase the ability of the suggested algorithms to identify the topology underlying a set of time series in an interpretable way, a sparsity-inducing penalty has been added to the estimation cost function. Two different approaches to the estimation of the model parameters are proposed, one of them (formulation A) based on 
minimizing the MSE in the sensor measurement space, and the other one (formulation B) based on minimising the MSE in the latent space. The solvers for both
techniques combine proximal gradient descent and projected gradient descent. Formulation B additionally requires to stabilize the mean and variance of the signals in the latent space, the associated constraints being enforced via a primal-dual algorithm.

Numerical results obtained from experiments that use  both synthetic and real data  indicate that the proposed technique achieves competitive results as its performance is compared with existing state-of-the-art models, in terms of topology identification and prediction ability. This shows that the proposed formulations are useful for determining the nonlinear relationships of sensor networks in the real world, encouraging further research in nonlinear VAR-based topology identification algorithms. Based on the information and experiments provided, it appears that formulation B is more suitable for estimating the adjacency graph, while formulation A is more efficient for prediction tasks.


%% file: IEEE_Journal_TSP/appendixA.tex
\section*{Appendix A}
In this appendix we provide the detailed derivation of the backward equations.
The gradient of the cost is obtained by applying the chain rule as follows:
\begin{equation} \label{chainrule1}
\begin{array}{c}
\frac{d C[t]}{d \theta_{i}}=\sum_{n=1}^{N} \frac{\partial C}{\partial \hat{z}_{n}[t]}  \frac{\hat{z}_{n}[t]}{\partial \theta_{i}} \\
\text { where } \frac{\partial C}{\partial \hat{z}_{n}[t]}=2(\hat{z}_n[t]-z_n[t]) = S_n
\end{array}
\end{equation}

\begin{equation} \label{chainrule2}
\frac{\partial \hat{z}_{n}[t]}{\partial \theta_{i}}=\frac{\partial f_{n}}{\partial \hat{y}_{n}}  \frac{\partial \hat{y}_{n}}{\partial \theta_{i}}+\frac{\partial f_{n}}{\partial \theta_{n}}  \frac{\partial \theta_{n}}{\partial \theta_{i}} \\
\quad 
\end{equation}
$$
\text { where } \frac{\partial \theta_{n}}{\partial \theta_{i}}=\left\{
\begin{array}{l}
I, n = i \\
0, n  \neq i
\end{array}\right.
$$
Substituting equation \eqref{chainrule1} into \eqref{chainrule2} yields
\begin{equation} \label{derivativecost1}
\frac{d C[t]}{d \theta_{i}}=\sum_{n=1}^{N} S_{n}\left(\frac{\partial f_{n}}{\partial \hat{y}_{n}}  \frac{\partial \hat{y}_{n}}{\partial \theta_{i}}+\frac{\partial f_{n}}{\partial \theta_{n}}  \frac{\partial \theta_{n}}{\partial \theta_{i}}\right)_\cdot
\end{equation}
Equation\eqref{derivativecost1} can be simplified as:

\begin{equation} \label{derivativecost2}
    \frac{d C[t]}{d \theta_{i}}
    =
    S_{i}
    \frac{\partial f_{i}}{\partial \theta_{i}}
    +\sum_{n=1}^{N} S_{n}\frac{\partial f_{n}}{\partial \hat{y}_{n}}  \frac{\partial \hat{y}_n}{\partial \theta_{i}}.
\end{equation}

$\text { The next step is to derive } \frac{\partial \hat{y}_{n}}{\partial \theta_{i}} \text { and } \frac{\partial f_{i}}{\partial \theta_{i}}$ of  equation \eqref{derivativecost2}:

\begin{equation} \label{gradientyhat}
    \frac{\partial \hat{y}_{n}[t]}{\partial \theta_i}
    =
    \sum_{p=1}^{P}\sum_{j=1}^{N}  
        a_{n j}^{(p)} 
        \frac{\partial}{\partial \theta_{j}} \tilde{y}_{j}[t-p]
        \frac{\partial \theta_{j}}{\partial \theta_{i}} .
\end{equation}

With $f_{i}^{\prime}\left(\hat{y}\right)= 
\frac{\partial f_i\left(\hat{y}, \theta_{i}\right)}{\partial\left(\hat{y}\right)}, $  
expanding $\tilde{y}_j[t-p]$ in equation \eqref{gradientyhat} yields



\begin{align} \label{derivativecost4}
 \frac{d C[t]}{d \theta_{i}}= &S_{i}\left(\frac{\partial f_{i}}{\partial \theta_{i}}\right)\nonumber
 \\
 &+\sum_{n=1}^{N} S_{n}\left(f_{n}^{\prime}(\hat{y}_{n}[t]) \sum_{p=1}^{P} a_{n i}^{(p)} \frac{\partial}{\partial \theta_{i}} g_{i}\left(z_{i}[t-p],\theta_{i}\right)\right)
\end{align}

Here, the vector 
$$\frac{\partial f_i\left(\hat{y}, \theta_{i}\right)}{\partial \theta_{i}} 
= \left[
\frac{\partial f_i\left(\hat{y}, \theta_{i}\right)}{\partial \alpha_{i}} 
\frac{\partial f_i\left(\hat{y}, \theta_{i}\right)}{\partial w_{i}} 
\frac{\partial f_i\left(\hat{y}, \theta_{i}\right)}{\partial k_{i}} 
\frac{\partial f_i\left(\hat{y}, \theta_{i}\right)}{\partial b_{i}} \right]$$ can be obtained by standard or automated differentiation 


However, \eqref{derivativecost4} involves the calculation of $\frac{\partial g_i(z, \theta_{i})}{\partial \theta_{i}}$, which is not straightforward to obtain. Since $g_i(z)$ can be computed numerically, the derivative can be obtained by implicit differentiation, realizing that the composition of $f_i$ and $g_i$ remains invariant, so that its total derivative is zero:

\begin{equation} \label{dfwrttheta1}
\frac{d}{d \theta_{i}}\left[f_i\left(g_i\left(z, \theta_{i}\right), \theta_{i}\right)\right]=0
\end{equation}

\begin{equation} \label{dfwrttheta2}
\Rightarrow \frac{\partial f_i\left(g_i\left(z, \theta_{i}\right), \theta_{i}\right)}{\partial g\left(z, \theta_{i}\right)} \frac{\partial g\left(z, \theta_{i}\right)}{\partial \theta_{i}}+\left.\frac{\partial f_i\left(\tilde{y}, \theta_{i}\right)}{\partial \theta_{i}}\right\vert_{\tilde{y} = g_i\left(z, \theta_{i}\right)}=0
\end{equation}

\begin{equation} \label{dfwrttheta3}
\Rightarrow {f^{\prime}_i(g_i(z,\theta_{i}))} \frac{\partial g\left(z, \theta_{i}\right)}{\partial \theta_{i}}+\left.\frac{\partial f_i\left(\tilde{y}, \theta_{i}\right)}{\partial \theta_{i}}\right\vert_{\tilde{y} = g_i\left(z, \theta_{i}\right)}=0
\end{equation}

\begin{align} \label{dgwrttheta1}
\text { Hence }& \frac{\partial g_i\left(z, \theta_{i}\right)}{\partial \theta_{i}}=\nonumber
\\
&-\big\{f^{\prime}_i(g_i(z,\theta_{i}))\big\}^{-1}{\left(\left.\frac{\partial f_i\left(\tilde{y}, \theta_{i}\right)}{\partial \theta_{i}}\right\vert_{\tilde{y} = g_i\left(z, \theta_{i}\right)}\right)}_\cdot 
\end{align}
%
%

The gradient of $C_T$ w.r.t. the VAR coefficient $a^{(p)}_{ij}$ is calculated as follows: 

\begin{equation} \label{dCwrtthetaa}
\frac{d C[t]}{d a^{(p)}_{i j}}=\sum_{n=1}^{N} S_{n}  \frac{\partial f_{n}}{\partial \hat y_{n}}  \frac{\partial \hat{y}_{n}}{\partial a_{i j}^{(p)}}
\end{equation}

\begin{equation} \label{dCwrtthetaa1} \nonumber
\frac{\partial \hat{y}_{n}[t]}{\partial a_{i j}^{(p)}}=\frac{\partial}{\partial a_{i j}^{(p)}} \sum_{p^\prime=1}^{P} \sum_{q=1}^{N} a_{n q}^{(p^\prime)} \tilde{y}_{q}[t-p]
\end{equation}

\begin{equation} \label{dCwrtthetaa2}
\begin{array}{c}
\text { where } 
\frac{\partial a_{n q}^{(p^\prime)}}{\partial a_{i j}^{(p)}}=\left\{\begin{array}{l}
1, n=i, p = p^\prime, \text { and } q=j \\
0, \text {otherwise}
\end{array}\right.
\end{array}
\end{equation}


\begin{equation} \label{dCwrtthetaa4}
\frac{d C[t]}{d a_{i j}^{(p)}}=S_i f_{i}^{\prime}\left(\hat{y}_{i}[t]\right) \tilde{y}_{j}[t-p]_\cdot 
\end{equation}


%% file: IEEE_Journal_TSP/appendixB.tex
\section*{Appendix B}

Consider $\check{f}$ such that 

\begin{equation} \label{transform_f1}
\check{f}_i =  \check{b}_{i} + \sum_{j=1}^{M} \check{\alpha}_{ij} h\left(\check{w}_{ij}y_{i}-\check{k}_{ij}\right)
\end{equation}
$\check{f}_i(1) = 1, \check{f}_i(-1) = -1$, $\check{f}_i(x) = x$. 
where $\check{\alpha}_i, \check{w}_i, \check{k}_i \text{ and } \check{b}_i $ are the learned parameters corresponding to $\check{f}_i$. A new function $\check{f}^{1}$ is defined such that 
\begin{equation} \label{transform_f2}
\check{f}_i^{1} =  \check{b}_{i}^{1} + \sum_{j=1}^{M} \check{\alpha}_{ij}^{1}h\left(\check{w}_{ij}^{1}y_{i}-\check{k}_{ij}^{1}\right)
\end{equation}
\begin{equation} \label{f2_c1}
\check{f}_i^{1}(\ubar{z}) =  \check{f}_i(-1) \text{ and } \check{f}_i^{1}(\bar{z}) =  \check{f}_i(1)
\end{equation}
\begin{equation} \label{f2_c3}
\check{f}_i^{1}(ax+B) =  \check{f}_i(ax+B)
\end{equation}
from \eqref{transform_f1} and \eqref{f2_c3} $\check{w}_{i}^{1} = a\check{w}_{i}$ and  $\check{k}_{i}^{1} = a\check{w}_{i}B + \check{k}_i$.
from equation \eqref{f2_c1} and \eqref{f2_c3},
\begin{equation} \label{f2_c4}
a\ubar{z} + B = -1  \text{ and }  a\bar{z} + B = 1 
\end{equation}

from \eqref{f2_c4} $a = -2/(\ubar{z} - \bar{z})$ and $B = 2\bar{z}/(\ubar{z} - \bar{z})$
Let 
\begin{equation} \label{transform_f3}
\check{f}_i^{2} =  \check{b}_{i}^{2} + \sum_{j=1}^{M} \check{\alpha}_{ij}^{2}h\left(\check{w}_{ij}^{2}y_{i}-\check{k}_{ij}^{2}\right)
\end{equation}
such that
\begin{equation} \label{f3_c1}
\check{f}_i^{2}(\bar{z}) =  \bar{z} \text{ , } \check{f}_i^{2}(\ubar{z}) =  \ubar{z} \text{ and } \check{f}_i^{2}(x) =  c\check{f}_i^{1}(x) +d 
\end{equation}

from \eqref{transform_f3} $\check{b}_i^2 = c\check{b}_i + d$ and $\check{\alpha}_i^2 = c\check{\alpha}_i$. From \eqref{f3_c1}
\begin{equation} \label{f3_c4}
\ubar{z} = -c + d  \text{ and }  \bar{z} = c + d
\end{equation}

from \eqref{f3_c4} $d = (\ubar{z} + \bar{z})/2 \text{ and } c = (\bar{z} - \ubar{z})/2$

Hence $\check{\alpha}_i = c \alpha_i, \check{b}_i =c b_i+d, \check{w}_i = aw_i,\check{k}_i = -w_i B+k_i$ where $c = (\bar{z} - \ubar{z})/2, d = (\bar{z} + \ubar{z})/2, a = -2/(\ubar{z} - \bar{z})$ and $ B = 2\bar{z}/(\ubar{z}-\bar{z})$.

%% file: main.bbl
\begin{thebibliography}{10}
\providecommand{\url}[1]{#1}
\csname url@samestyle\endcsname
\providecommand{\newblock}{\relax}
\providecommand{\bibinfo}[2]{#2}
\providecommand{\BIBentrySTDinterwordspacing}{\spaceskip=0pt\relax}
\providecommand{\BIBentryALTinterwordstretchfactor}{4}
\providecommand{\BIBentryALTinterwordspacing}{\spaceskip=\fontdimen2\font plus
\BIBentryALTinterwordstretchfactor\fontdimen3\font minus
  \fontdimen4\font\relax}
\providecommand{\BIBforeignlanguage}[2]{{%
\expandafter\ifx\csname l@#1\endcsname\relax
\typeout{** WARNING: IEEEtran.bst: No hyphenation pattern has been}%
\typeout{** loaded for the language `#1'. Using the pattern for}%
\typeout{** the default language instead.}%
\else
\language=\csname l@#1\endcsname
\fi
#2}}
\providecommand{\BIBdecl}{\relax}
\BIBdecl

\bibitem{luism}
L.~M. Lopez-Ramos, K.~Roy, and B.~Beferull-Lozano, ``Explainable nonlinear
  modelling of multiple time series with invertible neural networks,'' 2021.

\bibitem{roy2022joint}
K.~Roy, L.~M. Lopez-Ramos, and B.~Beferull-Lozano, ``Joint learning of topology
  and invertible nonlinearities from multiple time series,'' in \emph{2022 2nd
  International Seminar on Machine Learning, Optimization, and Data Science
  (ISMODE)}.\hskip 1em plus 0.5em minus 0.4em\relax IEEE, 2022, pp. 483--488.

\bibitem{giannakis2018topology}
G.~B. {Giannakis}, Y.~{Shen}, and G.~V. {Karanikolas}, ``Topology
  identification and learning over graphs: Accounting for nonlinearities and
  dynamics,'' \emph{Proceedings of the IEEE}, vol. 106, no.~5, pp. 787--807,
  2018.

\bibitem{dong}
\BIBentryALTinterwordspacing
X.~Dong, D.~Thanou, M.~Rabbat, and P.~Frossard, ``Learning graphs from data: A
  signal representation perspective,'' \emph{IEEE Signal Processing Magazine},
  vol.~36, no.~3, p. 44–63, May 2019. [Online]. Available:
  \url{http://dx.doi.org/10.1109/MSP.2018.2887284}
\BIBentrySTDinterwordspacing

\bibitem{zaman2020online}
B.~Zaman, L.~M. Lopez-Ramos, D.~Romero, and B.~Beferull-Lozano, ``Online
  topology identification from vector autoregressive time series,'' \emph{IEEE
  Transactions on Signal Processing}, 2020.

\bibitem{lopez2018dynamic}
L.~M. Lopez-Ramos, D.~Romero, B.~Zaman, and B.~Beferull-Lozano, ``Dynamic
  network identification from non-stationary vector autoregressive time
  series,'' in \emph{2018 IEEE Global Conference on Signal and Information
  Processing (GlobalSIP)}.\hskip 1em plus 0.5em minus 0.4em\relax IEEE, 2018,
  pp. 773--777.

\bibitem{chatarjee}
A.~Chatterjee, R.~J. Shah, and S.~Sen, ``Pattern matching based algorithms for
  graph compression,'' in \emph{2018 Fourth International Conference on
  Research in Computational Intelligence and Communication Networks (ICRCICN)},
  2018, pp. 93--97.

\bibitem{grangerc}
\BIBentryALTinterwordspacing
C.~W.~J. Granger, ``Investigating causal relations by econometric models and
  cross-spectral methods,'' \emph{Econometrica}, vol.~37, no.~3, pp. 424--438,
  1969. [Online]. Available: \url{http://www.jstor.org/stable/1912791}
\BIBentrySTDinterwordspacing

\bibitem{tirsobakth}
B.~Zaman, L.~Lopez-Ramos, D.~Romero, and B.~Beferull-Lozano, ``Online topology
  identification from vector autoregressive time series,'' \emph{IEEE
  Transactions on Signal Processing}, vol.~PP, pp. 1--1, 12 2020.

\bibitem{sparsitybasu}
J.~Lin and G.~Michailidis, ``Regularized estimation and testing for
  high-dimensional multi-block vector-autoregressive models,'' \emph{Journal of
  Machine Learning Research}, vol.~18, 08 2017.

\bibitem{tank2021neural}
A.~Tank, I.~Covert, N.~Foti, A.~Shojaie, and E.~B. Fox, ``Neural granger
  causality,'' \emph{IEEE Transactions on Pattern Analysis \& Machine
  Intelligence}, no.~01, pp. 1--1, mar 2021.

\bibitem{grangercausalitygoebel}
R.~Goebel, A.~Roebroeck, D.~Kim, and E.~Formisano, ``Investigating directed
  cortical interactions in time-resolved fmri data using vector autoregressive
  modeling and granger causality mapping,'' \emph{Magnetic resonance imaging},
  vol.~21, pp. 1251--61, 01 2004.

\bibitem{giannakis}
G.~Giannakis, Y.~Shen, and G.~Karanikolas, ``Topology identification and
  learning over graphs: Accounting for nonlinearities and dynamics,''
  \emph{Proceedings of the IEEE}, vol. 106, pp. 787--807, 05 2018.

\bibitem{ioannidis2019semiblind}
V.~N. Ioannidis, Y.~Shen, and G.~B. Giannakis, ``Semi-blind inference of
  topologies and dynamical processes over dynamic graphs,'' \emph{IEEE
  Transactions on Signal Processing}, vol.~67, no.~9, pp. 2263--2274, 2019.

\bibitem{tank}
\BIBentryALTinterwordspacing
A.~Tank, I.~Covert, N.~Foti, A.~Shojaie, and E.~B. Fox, ``Neural granger
  causality,'' \emph{IEEE Transactions on Pattern Analysis and Machine
  Intelligence}, p. 1–1, 2021. [Online]. Available:
  \url{http://dx.doi.org/10.1109/TPAMI.2021.3065601}
\BIBentrySTDinterwordspacing

\bibitem{Marinazzo}
D.~Marinazzo, W.~Liao, H.~Chen, and S.~Stramaglia, ``Nonlinear connectivity by
  granger causality,'' \emph{NeuroImage}, vol.~58, pp. 330--8, 09 2011.

\bibitem{stephan}
K.~Stephan, L.~Kasper, L.~Harrison, J.~Daunizeau, H.~Den~Ouden, M.~Breakspear,
  and K.~Friston, ``Nonlinear dynamic causal models for fmri,''
  \emph{NeuroImage}, vol.~42, pp. 649--62, 05 2008.

\bibitem{shen2018online}
Y.~Shen and G.~B. Giannakis, ``Online identification of directional graph
  topologies capturing dynamic and nonlinear dependencies,'' in \emph{2018 IEEE
  Data Science Workshop (DSW)}, 2018, pp. 195--199.

\bibitem{money2021online}
R.~Money, J.~Krishnan, and B.~Beferull-Lozano, ``Online non-linear topology
  identification from graph-connected time series,'' \emph{arXiv preprint
  arXiv:2104.00030}, 2021.

\bibitem{rohanjoshin}
------, ``Online non-linear topology identification from graph-connected time
  series,'' in \emph{2021 IEEE Data Science and Learning Workshop (DSLW)},
  2021, pp. 1--6.

\bibitem{shen2019nonlinear}
Y.~{Shen}, G.~B. {Giannakis}, and B.~{Baingana}, ``Nonlinear structural vector
  autoregressive models with application to directed brain networks,''
  \emph{IEEE Transactions on Signal Processing}, vol.~67, no.~20, pp.
  5325--5339, 2019.

\bibitem{bussmann2020neural}
B.~Bussmann, J.~Nys, and S.~Latr{\'e}, ``Neural additive vector autoregression
  models for causal discovery in time series,'' in \emph{Discovery Science:
  24th International Conference, DS 2021, Halifax, NS, Canada, October 11--13,
  2021, Proceedings 24}.\hskip 1em plus 0.5em minus 0.4em\relax Springer, 2021,
  pp. 446--460.

\bibitem{lutkepohl2005}
H.~Lütkepohl, \emph{New Introduction to Multiple Time Series Analysis}.\hskip
  1em plus 0.5em minus 0.4em\relax Springer, 2005.

\bibitem{lorentz96}
H.~Elshoush, B.~Al-Tayeb, and K.~Obeid, ``Enhanced serpent algorithm using
  lorenz 96 chaos-based block key generation and parallel computing for rgb
  image encryption,'' \emph{PeerJ Computer Science}, vol.~7, p. e812, 12 2021.

\bibitem{high_dim_lozano}
A.~Lozano, N.~Abe, Y.~Liu, and S.~Rosset, ``Grouped graphical granger modeling
  for gene expression regulatory network discovery,'' \emph{Bioinformatics
  (Oxford, England)}, vol.~25, pp. i110--8, 07 2009.

\bibitem{daubechies2004iterative}
I.~Daubechies, M.~Defrise, and C.~De~Mol, ``An iterative thresholding algorithm
  for linear inverse problems with a sparsity constraint,''
  \emph{Communications on Pure and Applied Mathematics: A Journal Issued by the
  Courant Institute of Mathematical Sciences}, vol.~57, no.~11, pp. 1413--1457,
  2004.

\bibitem{blondel2014large}
M.~Blondel, A.~Fujino, and N.~Ueda, ``Large-scale multiclass support vector
  machine training via euclidean projection onto the simplex,'' in \emph{2014
  22nd International Conference on Pattern Recognition}.\hskip 1em plus 0.5em
  minus 0.4em\relax IEEE, 2014, pp. 1289--1294.

\bibitem{boyd_dual}
S.~Boyd, N.~Parikh, E.~Chu, B.~Peleato, and J.~Eckstein, ``Distributed
  optimization and statistical learning via the alternating direction method of
  multipliers,'' \emph{Foundations and Trends in Machine Learning}, vol.~3, pp.
  1--122, 01 2011.

\end{thebibliography}
